\documentclass[preprint,11pt]{elsarticle}

\makeatletter
\def\ps@pprintTitle{%
  \let\@oddhead\@empty
  \let\@evenhead\@empty
  \def\@oddfoot{}%
  \def\@evenfoot{}%
}
\makeatother

\usepackage{amssymb}
\usepackage{amsmath}
\usepackage{amsthm}
\usepackage[dvipsnames]{xcolor}
\usepackage{algorithm}
\usepackage{algpseudocode}
\usepackage{cleveref}
\usepackage{bbm}
\usepackage{tabularray}
\usepackage{subfig}
\usepackage{dirtytalk}
\usepackage{mathtools}

\usepackage{comment}

\newtheorem{definition}{Definition}
\newtheorem{theorem}{Theorem}
\newtheorem{lemma}{Lemma}

\date{}

\theoremstyle{plain}
\newtheorem{assumption}{Assumption}

\newcommand{\R}{\mathbb{R}}
\newcommand{\N}{\mathbb{N}}
\newcommand{\M}{\mathcal{M}}
\newcommand{\G}{\mathcal{G}}
\newcommand{\A}{\mathcal{A}}
\newcommand{\U}{\mathcal{U}}
\newcommand{\D}{\mathcal{D}}

\newcommand{\ssubset}{\subset\joinrel\subset}
\DeclarePairedDelimiter\floor{\lfloor}{\rfloor}

\begin{document}

\begin{frontmatter}





\title{A Learning-based Domain Decomposition Method}

\author[label1]{Rui Wu}
\author[label3]{Nikola Kovachki}
\author[label1]{Burigede Liu\corref{cor1}}
\ead{bl377@cam.ac.uk}
\cortext[cor1]{Corresponding Author}

\affiliation[label1]{organization={Department of Engineering, University of Cambridge},
            addressline={Trumpington Street},
            city={Cambridge},
            postcode={CB2 1PZ},
            country={UK}}

\affiliation[label3]{organization={NVIDIA Corporation},
            addressline={2788 San Tomas Expy},
            city={Santa Clara},
            postcode={CA 95051},
            state={California},
            country={USA}}

\begin{abstract}
Recent developments in mechanical, aerospace, and structural engineering have driven a growing need for efficient ways to model and analyze structures at much larger and more complex scales than before. While established numerical methods like the Finite Element Method remain reliable, they often struggle with computational cost and scalability when dealing with large and geometrically intricate problems. In recent years, neural network-based methods have shown promise because of their ability to efficiently approximate nonlinear mappings. However, most existing neural approaches are still largely limited to simple domains, which makes it difficult to apply to real-world PDEs involving complex geometries. In this paper, we propose a learning-based domain decomposition method (L-DDM) that addresses this gap. Our approach uses a single, pre-trained neural operator—originally trained on simple domains—as a surrogate model within a domain decomposition scheme, allowing us to tackle large and complicated domains efficiently. We provide a general theoretical result on the existence of neural operator approximations in the context of domain decomposition solution of abstract PDEs. We then demonstrate our method by accurately approximating solutions to elliptic PDEs with discontinuous microstructures in complex geometries, using a physics-pretrained neural operator (PPNO). Our results show that this approach not only outperforms current state-of-the-art methods on these challenging problems, but also offers resolution-invariance and strong generalization to microstructural patterns unseen during training.

\end{abstract}




\end{frontmatter}


\section{Introduction}
The design and analysis of complex systems in mechanical, aerospace, and structural engineering \cite{bullo2019geometric, leng2024multidisciplinary, bauchau2009structural, umer2025approximate, stepinski2013advanced} are fundamentally governed by partial differential equations (PDEs), with a persistent drive to model systems at unprecedentedly larger and more complicated scales. Elliptic PDEs \cite{gilbarg1977elliptic}, in particular, are ubiquitous for modeling steady-state behaviors. For decades, the Finite Element Method (FEM) \cite{zienkiewicz2005finite, huebner2001finite} has been the gold standard for solving these equations. However, its practical utility is often hampered by prohibitive computational costs and poor scalability, a bottleneck that severely limits the fidelity of large-scale simulations with intricate geometries.

Over the past decade, there has been enormous interest in developing neural network-based approaches for solving PDEs. Two main streams of techniques have emerged. The first approach is Physics-Informed Neural Networks (PINNs) \cite{raissi2019physics}. At the heart of the PINN approach is the use of a neural network to approximate the unknown solution field, instead of a polynomial approximation employed by classical FEMs. This has opened up a line of research combining neural network basis functions with ideas from finite elements, including the Deep Ritz Method  \cite{yu2018deep} and variational PINN \cite{kharazmi2019variational}. However, training PINNs remains challenging, and as a solver, PINNs require retraining whenever there is a change in boundary conditions or PDE parameterization, which is similar to classical solvers.

The second stream of research focuses on Neural Operators (NOs) \cite{kovachki2023neural}. The idea of NOs is to directly approximate the solution operator for a parameterized PDE in function spaces, making it possible to solve not only one but a class of PDEs once trained. Numerous efforts have been developed, including DeepONet \cite{lu2019deeponet}, Fourier Neural Operator (FNO) \cite{li2020fourier}, Graph Kernel Network (GKN) \cite{li2020neural}, and more recent architectures such as Physics-Informed Neural Operator (PINO) \cite{li2024physics}, Geometry-Informed Neural Operator (GINO) \cite{li2023geometry}, Transformer Operator \cite{calvello2024continuum}, and Neural Interpretable PDEs (NIPS) for simultaneous solving and physics discovery via attention mechanisms  \cite{liu2025neural}. A key limitation of NOs is generalizing the trained operator to arbitrary complex domains. It is conceivable to achieve this by constructing an operator that maps the domain directly to the solution; however, this might be computationally prohibitive or require large amounts of data that is not usually available. The majority of work done so far has focused on training NOs on a fixed domain with varying PDE parameterizations and boundary conditions \cite{goswami2023physics, li2023fourier}.

To address this challenge, recent work has explored hybrid approaches that combine machine learning (ML) with established numerical techniques. A particularly promising direction integrates neural networks with domain decomposition methods (DDMs) \cite{schwarz1870ueber, lions1988schwarz, smith1997domain, toselli2004domain}---a class of algorithms first introduced by Schwarz in 1870 for solving elliptic boundary value problems on composite domains. The fundamental strength of DDMs lies in their divide-and-conquer strategy, where the global domain is partitioned into smaller, computationally tractable subdomains that are solved iteratively with coordinated interface conditions. This framework has evolved into powerful computational strategies, such as the classical Additive Schwarz Algorithm (ASA) \cite{courant1987additive}, with its inherent parallelism and modular structure naturally lending itself to combination with ML to address large-scale problems.

The integration of ML with DDMs has developed along two primary paradigms \cite{klawonn2024machine, heinlein2021combining}. In the first paradigm, Jagtap et al. \cite{jagtap2020conservative} proposed conservative PINN (cPINN). The key difference compared to the classical PINN approach is that they assigned different neural networks to each subdomain and enforce continuity across the interfaces of subdomains. A global optimization (training) is then performed via gradient descent to optimize the solutions of all subdomains simultaneously in order to find the global solution. Jagtap et al. further extended this idea \cite{jagtap2020extended} and proposed a space-time decomposition method called extended PINN (xPINN). Other more sophisticated architectures employing multiple specialized PINNs have also emerged, including GatedPINN \cite{stiller2020large}, MoE-PINN \cite{bischof2022mixture}, and APINN \cite{hu2023augmented}, which operate on an ensemble-of-experts principle. Additionally, architectures like DPINN \cite{dwivedi2019distributed} and PECANN \cite{basir2023generalized} incorporate interface conditions explicitly into the loss function to enforce solution consistency across subdomains. FBPINN \cite{moseley2023finite, dolean2024multilevel} has drawn inspiration from finite element methods to develop compact basis function representations. A key limitation in this paradigm lies in its reliance on global optimization during training, which tightly couples all subdomains and can lead to significant communication overhead and scalability challenges in distributed computing environments.

These limitations motivate a shift toward the second paradigm: using machine learning to enhance traditional DDM solvers. In this approach, the computational domain is first decomposed into subdomains, after which an ML model acts as a local solver for each region. The solutions are then iteratively refined through interface updates, mirroring classical DDM frameworks. Notable implementations include D3M \cite{li2019d3m}, which employs Deep Ritz networks as subdomain solvers, and DeepDDM \cite{li2020deep}, utilizing PINNs at each subdomain level. However, current methods face two critical challenges: (i) the necessity to train separate networks for different subdomains constrains model capacity due to computational resource limitations, and (ii) repeated retraining during iterative updates severely diminishes computational efficiency. These issues stem from the prevailing per-instance learning paradigm, where networks remain specialized to specific geometric configurations, boundary conditions, or decomposition strategies. Such specialization squanders the inherent generalizability of deep learning architectures and is particularly inefficient for problems involving numerous subdomains or complex geometries. This observation naturally leads to an alternative strategy: employing a single, pretrained neural network as a universal solver across all subdomains. Neural operators emerge as particularly suited for this role, as their ability to learn mappings between function spaces enables generalization across diverse parametric conditions, boundary values, and domain configurations—--precisely the flexibility required for robust DDM implementations.

Building upon these insights, this paper develops a learning-based domain decomposition method (L-DDM) where we employ pre-trained neural operators acting as surrogate models for the local subdomains and employ the classical iterative solver to reconstruct the solution for the entire domain. This algorithm leverages the function-space learning capabilities of neural operators and the parallel efficiency of the overlapping Schwarz method, creating a robust approach that follows a principled two-phase framework: 
\begin{enumerate}
    \item in the offline phase, the neural operator is pretrained on synthetic data to learn the solution manifold of the parametrized PDE, acquiring the ability to generalize to unseen boundary conditions and material microstructures; 
    \item in the online phase, the pretrained operator serves as a universal local solver within the Additive Schwarz Algorithm, where the computational domain is decomposed into overlapping subdomains and solutions are iteratively refined through interface updates.
\end{enumerate}

Specifically, we propose the physics-pretrained neural operator (PPNO), a neural operator architecture that learns the solution mappings for PDEs from both parameterization and boundary conditions while maintaining physical consistency.

This combined approach is computationally efficient, eliminating the need for retraining each new problem configuration. The PPNO's unique architecture preserves physical constraints while ensuring the strong generalization necessary for complex, large-scale problems. These capabilities are quantitatively validated in comprehensive benchmarks, where our algorithm accurately handles both in- and out-of-distribution microstructures under random boundary conditions, achieving a previously elusive combination of accuracy and robustness on spatially varying geometries.


The remainder fo this paper is organized as follows. In section 2, we first develop a general existence theorem for L-DDM on a general abstract PDE. We then demonstrate our approach on a elliptic PDE and provide the detailed numerical algorithm. In section 3, we propose the PPNO, an architecture specifically suitable for constrained physics and various parameterizations. Section 4 first examines the standalone performance of the PPNO, followed by large-scale numerical experiments demonstrating the algorithm's effectiveness across complex domains with heterogeneous material properties. Finally, section 5 discusses broader implications, limitations, and promising directions for future research in operator learning-enhanced domain decomposition methods.

\section{Problem Formulation and Theory}
\subsection{Abstract Dirichlet Problem}
We will first formulate our methodology on an abstract PDE problem, carefully outlining the conditions needed for the existence of an arbitrarily good neural operator approximation.

Let $\Omega \subset \mathbb{R}^d$ be a bounded, Lipschitz domain. We consider the abstract  Dirichlet problem of finding $u \in \U = \U (\Omega)$ such that
\begin{align}
\label{eq:abstract_problem}
    \begin{split}
    \mathcal{P}(a,u) &= 0, \quad \text{in } \Omega, \\
    u&= g, \quad \text{in } \partial \Omega
    \end{split}
\end{align}
for some $a \in \A$ and $g \in \U$ where $\mathcal{P}$ is a partial differential operator. We assume $\A$ and $\U$ are vector (or matrix) valued Banach function spaces on $\Omega$. Furthermore, we assume that $C^\infty (\bar{\Omega}) \subset \U$ and that there exists a bounded linear operator $T$ on $\U$ such that $T(g) = g|_{\partial \Omega}$ for every $g \in \U \cap C(\bar{\Omega})$, giving the appropriate interpretation to \eqref{eq:abstract_problem}. This assumption will be satisfied, for example, when $\U$ is a space of continuous or continuously differentiable functions or a Sobolev space with enough regularity to admit the existence of a trace operator \cite{adams2003sobolev}. We will further assume that \eqref{eq:abstract_problem} has a unique solution $u = u(a,g) \in \U$ for each pair $(a,g) \in \A \times \U$ and that the mapping $(a,g) \mapsto u$ is continuous. 

\subsection{Alternating Schwarz Method}
\label{subsec:alternating_schwarz}
The alternating Schwarz method constructs a sequence of approximate solutions to the global problem \eqref{eq:abstract_problem} by iteratively and sequentially solving a series of local problems on overlapping subdomains. For clarity of exposition, we will consider only two partitions of the domain $\Omega$, though the method generalizes to an arbitrary number.



Let $\Omega_1$, $\Omega_2$ be two Lipschitz domains which form an overlapping partition of $\Omega$, in particular, $\Omega_1 \cup \Omega_2 = \Omega$. We denote the intersection as $I = \Omega_1 \cap \Omega_2 \neq \emptyset$ and require it to have a positive Lebesgue measure. The non-overlapping regions are denoted $D_1 = \Omega_1 \setminus  \Omega_2$ and $D_2 = \Omega_2 \setminus \Omega_1$, which are also assumed to have positive Lebesgue measure. Finally we define part of the interior boundary sets $\Gamma_1 = \partial \Omega_1 \setminus \partial \Omega$ and $\Gamma_2 = \partial \Omega_2 \setminus \partial \Omega$. An example domain configuration is illustrated in Figure~\ref{fig:domain}.

\begin{figure}[h!]
    \centering
    \includegraphics[width=0.9\linewidth]{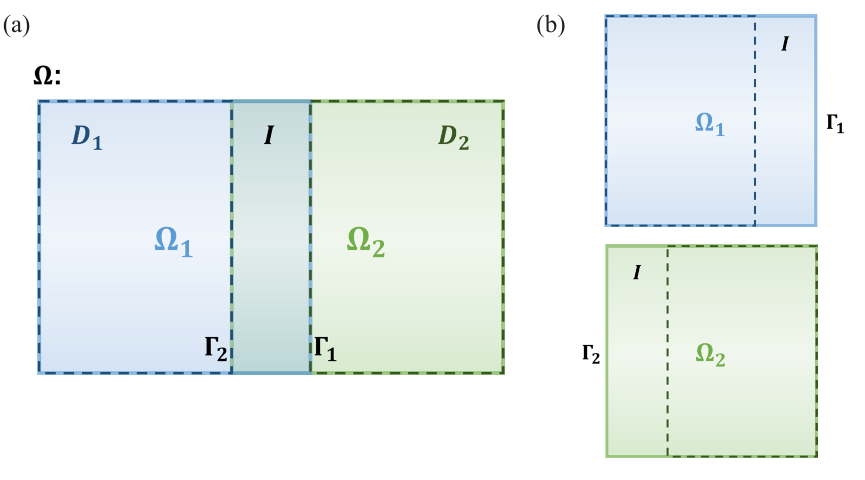}
    \caption{Schematic illustration of the overlapping domain decomposition. The global domain $\Omega$ is partitioned into two overlapping subdomains $\Omega_1$ and $\Omega_2$ such that $\Omega = \Omega_1 \cup \Omega_2$ and $I = \Omega_1 \cap \Omega_2$. The non-overlapping regions $D_1$ and $D_2$ are denoted using dash lines. The interior boundary portion $\Gamma_1 = \partial \Omega_1 \setminus \partial \Omega$ and $\Gamma_2 = \partial \Omega_2 \setminus \partial \Omega$ is also highlighted.}
    \label{fig:domain}
\end{figure}

We now consider the following pair of local problems
\begin{align}
\label{eq:abstract_problem_local_1}
    \begin{split}
    \mathcal{P}(a|_{\Omega_1},v) &= 0, \quad \text{in } \Omega_1, \\
    v &= g, \quad \text{in } \partial \Omega_1
    \end{split}
\end{align}
and
\begin{align}
\label{eq:abstract_problem_local_2}
    \begin{split}
    \mathcal{P}(a|_{\Omega_2},w) &= 0, \quad \text{in } \Omega_2, \\
    w &= g, \quad \text{in } \partial \Omega_2 \setminus \Gamma_2 \\
    w &= v, \quad \text{in } \Gamma_2
    \end{split}
\end{align}
and assume both have unique solutions $v \in \U (\Omega_1)$ and $w \in \U (\Omega_2)$. It is easy to see that \eqref{eq:abstract_problem_local_2} has the same structure as \eqref{eq:abstract_problem_local_1} since $v = g$ in $\partial I$. Hence unique solutions to \eqref{eq:abstract_problem_local_1} and \eqref{eq:abstract_problem_local_2} are guaranteed as long as \eqref{eq:abstract_problem} has a unique solution on any arbitrary Lipschitz domain. The idea is that $v$ is first solved in $\Omega_1$ with an \say{arbitrary} boundary condition in $\partial \Omega_1 \setminus \partial\Omega$ which we here simply take from the extension of $g$ on the whole domain. Then $w$ is solved in $\Omega_2$ by taking the \say{unknown} part of its boundary condition from $v$. An approximate solution on the whole domain is found by gluing $v$ and $w$ with a partition of unity as described below. We will abuse notation and continue to write $v$ and $w$ for their zero extensions to the whole of $\Omega$.

We will now define a partition of unity in $\bar{\Omega}$ which we will use to glue our local solutions.
Let $\varphi \in C^\infty(\bar{\Omega})$ be such that $\varphi \equiv 1$ on $\bar{D}_1$ and $\varphi \equiv 0$ on $\bar{D}_2$. Let $\phi \in C^\infty (\bar{\Omega})$ be such that $\phi \equiv 1$ on $\bar{D}_2$ and $\phi \equiv 0$ on $\bar{D}_1$. Furthermore we enforce that $\varphi(x) + \phi(x) = 1$ for all $x \in \bar{I}$. It is clear that $\varphi(x) + \phi(x) = 1$ for all $x \in \bar{\Omega}$ hence $\{\varphi, \phi\}$ is a partition of unity. In section~\ref{sec:numerics}, we give an explicit example of such a partition which is $C$-regular and note that $C^\infty$-regularity can be obtained by mollification. We can now define the base operator $\G$ of the alternating Schwarz method as $\G(a,g) = v \varphi + w \phi$. We make the assumption that this construction satisfies $\G(a,g) \in \U$ and remark that since our partition of unity is smooth and $v$ and $w$ agree on $\Gamma$, the assumption will be satisfied for $\U$ a Sobolev space or a space of continuous or continuously differentiable functions. We will use the following notation to denote the $n$-fold composition of $\G$ with itself, that is, for any $n \in \N$
\begin{equation}\label{eq:G-composition}
\G^{(n)}(a,g) = \underbrace{\G \big( a, \G(a, \dots ) \big )}_{n \text{ times}}
\end{equation}
so that $\G^{(1)}(a,g) = \G(a,g)$, $\G^{(2)}(a,g) = \G \big (a, \G(a, g) \big)$ and so on. We can now define the alternating Schwarz method.
\begin{definition}
    We say that the abstract problem \eqref{eq:abstract_problem} satisfies the \textit{alternating Schwarz method} if for any pair $(a,g) \in \A \times \U$ and any $\epsilon > 0$, there exists $N \in \N$ such that
    \begin{equation}
        \|\G^{(n)}(a,g) - u(a,g) \|_\U < \epsilon, \qquad \forall \: n \geq N.
    \end{equation}
\end{definition}

In the next section, we will approximate $\G$ with a neural operator over compact sets of $\A \times \U$. To ensure that our approximation will be close to the true solution of \eqref{eq:abstract_problem}, we will also need to control the sequences generated by the alternating Schwarz method over the set of parameters of the PDE. To that end, we make the following definition.
\begin{definition}
    We say that a set $K \subset \A \times \U$ has the \textit{uniform alternating Schwarz property} if for any $\epsilon > 0$, there exists $N \in \N$ such that
    \begin{equation}
        \sup_{(a,g) \in K} \|\G^{(n)}(a,g) - u(a,g) \|_\U < \epsilon, \qquad \forall \: n \geq N.
    \end{equation}
\end{definition}
We will now show that, if $\G^{(n)}$ is uniformly locally Lipschitz, then any compact set has the uniform alternating Schwarz property. The proof is given in \ref{sec:approximation_of_the_schwarz_method}.

\begin{lemma}
    \label{lemma:uniform_locally_lip}
    Let $K \subset \A \times \U$ be compact. Suppose that there exists a constant $L > 0$ such that, for any $n \in \N$ and $(a_1,g_1), (a_2,g_2) \in K$,
    \begin{equation}
        \|\G^{(n)}(a_1,g_1) - \G^{(n)}(a_2,g_2)\|_{\U} < L \big ( \|a_1 - a_2\|_\A + \|g_1 - g_2\|_\U \big ).
    \end{equation}
    Furthermore, assume that the map $(a,g) \mapsto u$ defined by \eqref{eq:abstract_problem} is Lipschitz on $K$. Then $K$ has the uniform additive Schwarz property.
\end{lemma}

\subsection{Neural Operator Approximation}
\label{sec:neural_op_approx}

We will now define a neural operator approximation to the alternating Schwarz method defined in the previous section and prove that our approximation can be arbitrarily good. Let $\D$ be either $\Omega_1$ or $\Omega_2$ and consider the problem 
\begin{align}
\label{eq:abstract_problem_local_D}
    \begin{split}
    \mathcal{P}(a|_{\D},z) &= 0, \quad \text{in } \D, \\
    z &= g, \quad \text{in } \partial \D
    \end{split}
\end{align}
for any $a \in \A$ and $g \in \U$. We denote by $\G_1 : \A \times \U \to \U(\Omega_1)$ and $\G_2 : \A \times \U \to \U(\Omega_2)$ the solution operators of \eqref{eq:abstract_problem_local_D} mapping $(a,g) \mapsto z$. We will consider their outputs as functions on $\Omega$ by zero extension. It is easy to see that
\begin{equation}
    \big ( \G_1 (a, g) \varphi + g \phi \big )(x) = \begin{cases}
    g(x), & x \in \Omega_2 \setminus \Gamma, \\
    \G_1 (a, g)(x), & x \in \Gamma. \end{cases}
\end{equation}
Hence we can write the base operator of the alternating Schwarz method as
\begin{equation}
\G(a,g) = \G_1 \big ( a,g \big ) \varphi + \G_2 \big (a, \G_1 (a,g) \varphi + g \phi \big) \phi.
\end{equation}
We now consider two neural operators $\hat{\G}_1 : \A \times \U \to \U(\Omega_1)$ and $\hat{\G}_2 : \A \times \U \to \U(\Omega_2)$
and define an \textit{alternating Schwarz} neural operator as 
\begin{equation}
\hat{\G}(a,g) = \hat{\G}_1 \big ( a,g \big ) \varphi + \hat{\G}_2 \big (a, \hat{\G}_1 (a,g) \varphi + g \phi \big) \phi.
\end{equation}
We can now state and prove our main result under the following assumption.
\begin{assumption}
    \label{asump:continuity}
    The operators $\G_j: \A \times \U \to \U(\Omega_j)$ are continuous and the mappings $\G_j(a,\cdot) : \U \to \U(\Omega_j)$ are locally Lipschitz continuous for every $a \in \A$ and $j=1,2$.  
\end{assumption}

\begin{theorem}
    \label{thm:main_approximation}
    Let $\U = C^s (\bar{\Omega})$ for some $s \in \N_0$. Suppose problem \eqref{eq:abstract_problem} satisfies the alternating Schwarz method and that assumption~\ref{asump:continuity} holds. Then for any compact set $K \subset \A \times \U$ satisfying the uniform alternating Schwarz property and $\epsilon > 0$, there exists a number $n \in \N$ and an alternating Schwarz neural operator $\hat{\G} : \A \times C^s (\bar{\Omega}) \to C^s (\bar{\Omega})$ such that
    \begin{equation}
        \sup_{(a,g) \in K} \|\hat{\G}^{(n)}(a,g) - u(a,g) \|_{C^s} < \epsilon.
    \end{equation}
\end{theorem}

The proof of the above result is given in \ref{sec:approximation_of_the_schwarz_method}. It relies on an approximation result about the composition of a neural operator with its own output. In particular, we show that for Lipchitz operators between spaces of continuously differentiable functions, a neural operator approximation can remain arbitrarily close after $n > 1$ compositions with itself. This result is of independent interest as it applies generally to all auto-regressive models which aim to approximate function-valued time series as is commonly done, for example, in weather forecasting \cite{pathak2022fourcastnet, lam2023learning}. Precise statements and proofs are given in \ref{sec:compisiton_of_operators}. The proof relies on the uniform boundedness property of the mollification operator defined in \cite[Appendix A]{lanthaler2023nonlocality} when acting between spaces of continuously differentiable functions. This uniform boundedness, however, does not hold when the mollification operator acts between Sobolev spaces hence the proof does not immediately generalize to this setting. We leave extension of the current results to operators acting between Sobolev spaces as future work.

\subsection{Second Order Elliptic PDE}
In the remainder of this paper, we illustrate the abstract framework by considering a linear, second-order elliptic boundary-value problem defined as
\begin{align}
\label{eq:second_order_elliptic_pde}
\begin{split}
    \mathcal{P}(a, u) = -\nabla \cdot (a \nabla u)  &= 0, \quad x \in \Omega, \\
    u &= g, \quad x \in \partial \Omega,
\end{split}
\end{align}
where \(a \in  L^{\infty} (\Omega; \mathbb{R}^{d \times d}_{\text{sym}, \succ 0})\) is a positive definite matrix field, $g \in H^{1/2}(\Omega)$ is a given boundary condition. This PDE is crucial in various scientific and engineering fields due to its ability to model essential physical phenomena. It finds applications in mechanics for understanding material deformation, in materials science for studying diffusion processes, in fluid flow for modeling potential flow, and in heat conduction for analyzing temperature distribution. 

Standard elliptic theory dictates that Eq.~\eqref{eq:second_order_elliptic_pde} has a unique solution $u \in H^{1}(\Omega)$ \cite{gilbarg2001elliptic}. If $d = 1$ then $H^1 (\Omega) \subset C(\bar{\Omega})$ and, if we further assume that $g \in C(\bar{\Omega})$, then the operator of interest maps between spaces for which Theorem~\ref{thm:main_approximation} holds. However, for $d > 1$, the Sobolev embedding no longer holds, hence a stronger assumption on the diffusion coefficient $a$ is needed to guarantee the continuity of $u$. For example, H{\"o}lder continuity of $a$ implies H{\"o}lder continuity of $u$ in $\Omega$ and continuity in $\bar{\Omega}$, see \cite[Chapter 6]{gilbarg2001elliptic}. In order to apply the theory established in the previous setting to this case, one needs to (i) establish continuity of the map $a \mapsto u$ for input and output space $C(\bar{\Omega})$, and (ii) show that the problem satisfies the alternating Schwarz method on $C(\bar{\Omega})$. While we believe that this is an interesting theoretical direction in the analysis of elliptic PDE(s), it is outside the scope of the current work.

\subsection{A Learning-based Domain Decomposition}
We now proceed to develop the learning-based domain decomposition method for Eq.~\eqref{eq:second_order_elliptic_pde}. The formulation of the subdomain problems is similar to those described in subsection~\ref{subsec:alternating_schwarz}, except that we now consider $M \in \N$ subdomains $\Omega_j$ that are congruent to a fixed domain $\mathcal{D}$. We consider a partition of unity $\{\phi_j(x)\}_{j=1}^M$, and the global solution $u(x)$ can be reconstructed from the local solutions $u_j(x)$ by:
\begin{equation}
u(x) = \sum_{j=1}^M \phi_j(x) u_j(x).
\end{equation}

In Theorem~\ref{thm:main_approximation}, we have established the existence of a neural operator approximation for the alternating Schwarz method. However, in practice, the alternating Schwarz method faces challenges in parallelization, which is essential for large, complicated domains. In the rest of the paper, we shall develop the learning-based domain decomposition method based on the additive Schwarz method. The fundamental distinction between these variants is illustrated in Figure~\ref{fig:schwarz_methods}. The alternating method is inherently sequential, as computing the solution on subdomain $\Omega_j$ at iteration $n+1$ requires the already updated solution from subdomain $\Omega_{j-1}$ within the same iteration. In contrast, the additive method computes the solutions for all subdomains at iteration $n+1$ concurrently, with each local solver depending only on the prediction of global solution from the previous iteration, $\hat{u}^{(n)}$. We consider the extended $g \in H^1(\Omega)$ (c.f. section~\ref{sec:extension_operator} for details), and let $\mathcal{G}_j$ denote the solution operator for each subdomain $\Omega_j$, which we extend to the full domain via zero extension. The base operator of the additive Schwarz method can be written as
\begin{equation}
\mathcal{G}(a,g) = \sum_{j=1}^{M} \mathcal{G}_j(a,g) \phi_j,
\end{equation}
and the $n$-fold composition of $\mathcal{G}$ with itself follows the definition given in \eqref{eq:G-composition}.
Our approximation theory focuses on the alternative method because it is the most widely analyzed and results exist for different PDEs and variations on the methods; see, for example, the seminal works \cite{lions1988onthe,lions1989onthe,lions1990onthe}. Similar results to our own can also be derived for the additive method. We do not pursue this in the current work as our goal is not an exhaustive analysis of different Schwarz-type methods, but rather a theoretical and practical demonstration of their potential use within deep learning algorithms.  

\begin{figure}[h!]
    \centering
    \includegraphics[width=1.0\linewidth]{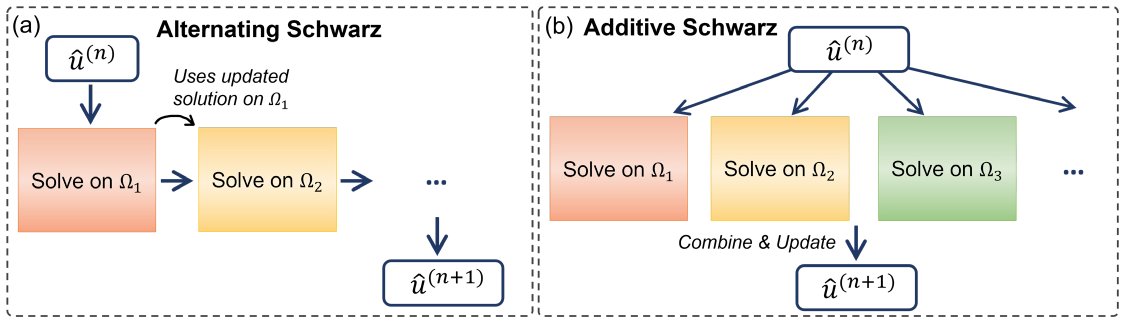}
    \caption{A schematic comparison of Schwarz methods for iteration $n \to n+1$. (a) Alternating Schwarz (Sequential): The update for subdomain $\Omega_j$ depends on the new solution from $\Omega_{j-1}$. (b) Additive Schwarz (Parallel): All subdomain problems are solved independently using the global solution from the previous iteration, $u^{(n)}$. Our algorithm implements the additive variant for computational efficiency.}
    \label{fig:schwarz_methods}
\end{figure}

We now proceed to describe the learning-based domain decomposition algorithm in detail. Note that, under the condition that each $\Omega_j$ is congruent to the fixed domain $\mathcal{D}$, the solution operator $\mathcal{G}_j$ is identical to $\mathcal{G}_k$ for all $j, k \in \{1, \dots, M\}$. We shall use $\mathcal{G}_\mathcal{D}$ to denote this common solution operator. Our goal is to first construct a neural operator approximation $\hat{\mathcal{G}}_\mathcal{D}$ for the local PDE and then apply it as a surrogate in the additive Schwarz algorithm.

In Algorithm~\ref{alg:domain_decomposition}, we detail our learning-based domain decomposition algorithm. We achieve this by first building the neural operator approximation of the solution operator for the local PDEs, $\hat{\mathcal{G}}_\mathcal{D}$. This is done by pretraining a neural operator with arbitrary functions $a_\mathcal{D} \in L^{\infty}(\mathcal{D}; \mathbb{R}^{d \times d}_{\text{sym}, \succ 0})$ and boundary conditions specified via extensions $g_\mathcal{D} \in H^{1}(\mathcal{D})$. We call this the offline training phase.

Next, we apply the additive Schwarz method, which we refer to as the online solution phase. We initialize $\hat{u}^{(0)}$ by extending $g$ to $\Omega$ using an extension operator detailed in Section~\ref{sec:extension_operator}. We then construct the neural operator approximation of the additive Schwarz algorithm as
\begin{equation}
\hat{\mathcal{G}}(a,g) = \sum_{j=1}^{M} \hat{\mathcal{G}}_{\mathcal{D},j}(a,g) \phi_j,
\end{equation}
where $\hat{\mathcal{G}}_{\mathcal{D},j}$ is defined as:
\begin{equation}
\hat{\mathcal{G}}_{\mathcal{D},j}(a,g)(x) = 
\begin{cases}
    \hat{\mathcal{G}}_\mathcal{D}(a|_{\Omega_j}, g|_{\Omega_j})(x), & x \in \Omega_j, \\
    0, & x \in \Omega \setminus \Omega_j.
\end{cases}
\end{equation}
Note that, since the subdomains are congruent to $\mathcal{D}$, we assume appropriate mappings (e.g., affine transformations) are applied to align the inputs with the trained domain $\mathcal{D}$. At iteration $n+1$, the solution is obtained via
\begin{align}
\hat{u}^{(n+1)}(x) = \sum_{j=1}^{M} \hat{u}_{j}^{(n+1)}(x) \phi_j(x) 
&= \sum_{j=1}^{M} \hat{\mathcal{G}}_{\mathcal{D},j}(a, u^{(n)}) \phi_j(x) \\ &= \underbrace{\hat{\mathcal{G}} \big( a, \hat{\mathcal{G}}(a, \dots ) \big)}_{n+1 \ \text{times}}(x).
\end{align}
We define the successive error as
\begin{equation} \label{eq:successive_error}
S = \|\hat{u}^{(n)} - \hat{u}^{(n-1)}\|_{L^2(\Omega)}.
\end{equation}
The iterative process is terminated when $S < \epsilon_{\text{tol}}$ for a prescribed tolerance $\epsilon_{\text{tol}} > 0$.

\begin{algorithm}[H]
\caption{Learning-based Domain Decomposition Algorithm}
\label{alg:domain_decomposition}
\begin{algorithmic}
\State \textbf{Offline pretraining (once for any microstructure pattern):}
\State Solve the PDE defined on a bounded domain $\mathcal{D}$ for various realizations of microstructures $a$ and extensions specifying boundary conditions $g$. 
\State Use this data to train the neural operator model $\hat{\mathcal{G}}_\mathcal{D}$. \\ 
\State \textbf{Online solution:}
\State \textit{Partition domain}: $\Omega = \bigcup_{j=1}^{M} \Omega_j$ into overlapping domains congruent to $\mathcal{D}$.
\State \textit{Initialize}: Extend $g$ to $\hat{u}^{(0)}$ on $\Omega$ (see Section~\ref{sec:extension_operator}).
\State $n = 0$
\While{$\|\hat{u}^{(n+1)} - \hat{u}^{(n)}\|_{L^2(\Omega)} > \epsilon_{\text{tol}}$}
\State $n \gets n + 1$
\State Compute the local solutions: for $j = 1$ to $M$
\begin{equation*}
    \hat{u}_{j}^{(n)} = \hat{\mathcal{G}}_{\mathcal{D},j}(a, \hat{u}^{(n-1)})
\end{equation*}
\State Assemble the global solution:
\begin{equation*} 
    \hat{u}^{(n)}(x) = \sum_{j=1}^M \hat{u}_{j}^{(n)}(x) \phi_j(x)
\end{equation*}
\EndWhile
\end{algorithmic}
\end{algorithm}

\section{Local Neural Operators}
\label{sec:local_neural_operators}
In this section we describe the design of our local neural operator \cite{liu2024neural}; a graphical depiction of the architecture is shown in Fig. \ref{fig:architecture}. The structure of this architecture is similar to that of a U-Net \cite{ronneberger2015u} but carefully designed so that no parameters are tied to any specific discretizations of the input data. This allows us to learn from data given at arbitrary discretizations. We will assume that input functions belong to the space $\mathcal{U}(\Omega;\R^{d_{in}})$ and corresponding outputs belong to $\mathcal{U}(\Omega;\R^{d_{out}})$.

\begin{figure}[h!]
    \centering
    \includegraphics[width=1.0\linewidth]{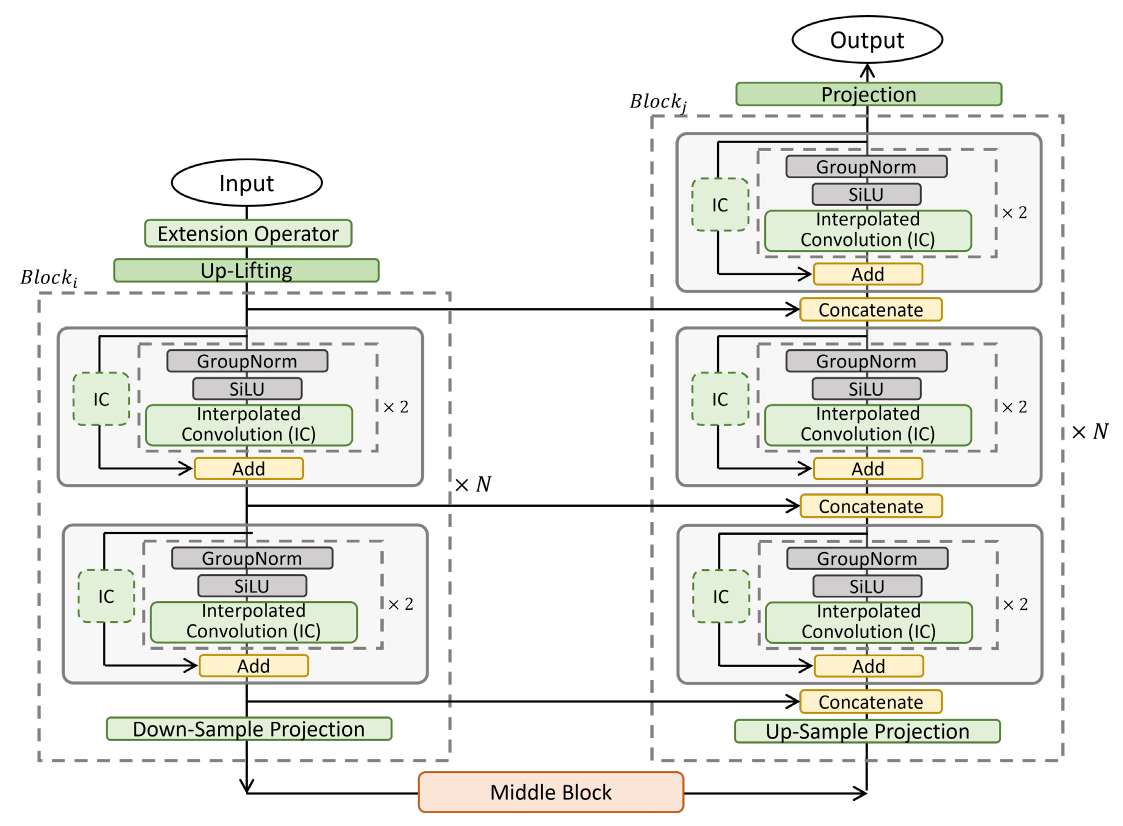}
    \caption{Schematic of the PPNO architecture. The workflow proceeds in three main stages: (1) An input lifting layer transforms the microstructure $a(x)$ and the extended boundary conditions $(\mathcal{V}g)(x)$ into a high-dimensional latent representation. (2) The core processing is performed by a series of blocks arranged in a U-shaped structure, each of which applies interpolated convolutions with a ResNet architecture to efficiently capture global dependencies. Downsampling and upsampling projectors are integrated into the contracting and expanding paths of the U-Net, respectively, to process features at multiple scales while preserving spatial consistency. (3) A final projection layer maps the high-dimensional representation back to the desired output dimension, yielding the solution field $\hat{u}$ and $\nabla_x\hat{u}$.}
    \label{fig:architecture}
\end{figure}

\subsection{Interpolated Convolution}

Convolutional neural networks have been widely successful in applications to visual and audio data \cite{Goodfellow-et-al-2016}. Their basic building block is a convolution operator with a kernel described by its values on a few set of fixed nodal points. The point-wise values of the kernel are tunable parameters that are learned from data via an optimization procedure. In this section, for simplicity, we will assume that $\Omega = (0,1)$ is the unit interval but note that the generalization to higher dimensions follows similarly. To make things concrete, let $\kappa : \R \to \R^{d_{out} \times d_{in}}$ be a matrix valued function supported on the interval $[-\frac{a}{2},\frac{a}{2}]$ for some fixed $0 < a \ll 1$. For a function $\omega: \Omega \to \R^{d_{in}}$, we abuse notation and continue to write $\omega$ for its zero extension to the whole of $\R$. We then define the convolution operator $\mathcal{F}$ as 
\begin{equation}
    \label{eq:convolution}
    \mathcal{F}(\omega)(x) = \int_{\R} \kappa(x - y)\omega(y) \: \mathsf{d}y
\end{equation}
for any $x \in \Omega$ so that $\mathcal{F}(\omega) : \Omega \to \R^{d_{out}}$. Fix some $K \in \N$ with $K > 1$ and define the set of points $x_j = -\frac{a}{2} + \frac{j-1}{K-1}a$ for $j=1,\dots,K$. These points form a uniform discretization of the interval $[-\frac{a}{2},\frac{a}{2}]$. We impose that $\kappa(x_j) = \kappa_j$ for each $j=1,\dots,K$ for some set of matrices $\kappa_1,\dots,\kappa_j \in \R^{d_{out} \times d_{in}}$ which constitute the parameters of the kernel $\kappa$. Suppose that the function $\omega$ is known on the set of points $y_j = \frac{j-1}{K-1}a$ for $j=1,\dots,s$ where $s = \frac{K-1}{a} + 1$, assuming that $a$ is chosen so that $s \in \N$. Then we approximate the integral as
\begin{equation}
    \mathcal{F}(\omega)(y_j) \approx \frac{1}{K} \sum_{l=1}^K \kappa (x_l) \omega(y_j - x_l).
\end{equation}

Suppose now that $\omega$ is given at an increased resolution, so that its values are known at the points $\tilde{y}_j = \frac{j-1}{2(K-1)}a$ for $j=1,\dots,\frac{2(K-1)}{a} + 1$. Applying the standard convolution operation that is commonly used in convolutional neural networks \cite{paszke2019pytorch} will result in computation of
\begin{equation}
    \frac{1}{K} \sum_{l=1}^K \kappa(x_l) \omega \big( \tilde{y}_j - (x_l /2) \big )
\end{equation}
which is clearly not a valid approximation of $\mathcal{F}(\omega)(\tilde{y}_j)$. The above sum may be interpreted, in the continuum, as an approximation of $\mathcal{F}(\omega)(\tilde{y}_j)$ with $\kappa$ supported on $[-\frac{a}{4}, \frac{a}{4}]$ and defined at the points $-\frac{a}{4} + \frac{j-1}{2(K-1)}a$ for $j=1,\dots,K$. In this way, $\mathcal{F}$  is not well-defined on any function space since it depends on a particular discretization of the input function $\omega$. This argument demonstrates why standard convolutional neural networks cannot be used with data given at different resolutions.

To define $\mathcal{F}$ appropriately on function space, we treat the point $\{x_j\}_{j=1}^K$ as nodal points and extend $\kappa$ to the whole of its support by a linear spline interpolation \cite{devore1993constructive}. This allows us to evaluate $\kappa$ at any point in its support at and thus evaluate $\mathcal{F}$ for different discretizations of $\omega$. In particular, we have
\begin{equation}
    \mathcal{F}(\omega)(\tilde{y}_j) \approx \frac{1}{2K} \sum_{l=1}^{2K} \kappa(\tilde{x}_l) \omega (\tilde{y}_j - \tilde{x}_l)
\end{equation}
where $\tilde{x}_j = -\frac{a}{2} + \frac{j-1}{2(K-1)}a$ for $j=1,\dots,2K-1$. This approximation can be efficiently implemented in a a differentiable way in standard deep learning frameworks such as Pytorch \cite{paszke2019pytorch} and is straightforwardly extended to higher dimensional domains. Our construction allows us to take advantage of optimization techniques for finding the parameters $\{\kappa_j\}_{j=1}^K$ that have been empirically successful for convolutional neural networks while also being able to utilize data at different discretizations.

\subsection{Single Layer}

Having appropriately defined the interpolated convolution, we can now define a single layer of our local neural operator.
In particular, we define the mapping
\begin{equation}
    \label{eq:one_layer}
    L(\omega) = \mathcal{F}_3(\omega) + (\mathcal{F}_2 \circ \sigma \circ \mathcal{N}_2 \circ \mathcal{F}_1 \circ \sigma \circ \mathcal{N}_1)(\omega)
\end{equation}
where $\sigma : \R \to \R$ is a non-linear activation function applied point wise, $\mathcal{F}_1, \mathcal{F}_2$ are two interpolated convolution operators with different parameters as previously defined, and $\mathcal{N}_1$, $\mathcal{N}_2$ are two group normalization operators with different parameters. If $\omega \in \mathcal{U}(\Omega;\R^{d_{in}})$ then $\mathcal{F}_1 : \mathcal{U}(\Omega;\R^{d_{in}}) \to \mathcal{U}(\Omega;\R^{d_{out}})$ and $\mathcal{F}_2 : \mathcal{U}(\Omega;\R^{d_{out}}) \to \mathcal{U}(\Omega;\R^{d_{out}})$. If $d_{in} = d_{out}$ then $\mathcal{F}_3 = I_{\mathcal{U}}$ is the identity on $\mathcal{U}$ and otherwise $\mathcal{F}_3 : \mathcal{U}(\Omega;\R^{d_{in}}) \to \mathcal{U}(\Omega;\R^{d_{out}})$ is an interpolated convolution operator with its own parameters. Therefore $L$ is the non-linear mapping $L: \mathcal{U}(\Omega;\R^{d_{in}}) \to \mathcal{U}(\Omega;\R^{d_{out}})$.

We will define the group normalization operator as follows. Let $\pi_1,\dots,\pi_n \subseteq \{1,\dots,d_{in}\}$ be a disjoint cover and let $k : \{1,\dots,d_{in}\} \to \{1,\dots,n\} $ be such that $l \in \pi_m$ implies $k(l) = m$ for any $l \in \{1,\dots,d_{in}\}$. For any $j \in \{1,\dots,n\}$, we define the mappings
\begin{align}
    m_{\pi_j}(\omega) &= \frac{1}{|\Omega| |\pi_j|} \sum_{l \in \pi_j} \int_{\Omega} \omega_l (x)\: \mathsf{d}x, \\
    s_{\pi_j}(\omega) &= \left ( \frac{1}{|\Omega| |\pi_j|} \sum_{l \in \pi_j} \int_{\Omega} |\omega_l (x) - m_{\pi_j}(\omega)|^2\: \mathsf{d}x \right )^{1/2} .
\end{align}
Then, for any $l \in \{1,\dots,d_{in}\}$, we define 
\begin{equation}
    \mathcal{N}(\omega)_l = \frac{\omega_l - m_{\pi_{k(l)}}(\omega)}{\epsilon + s_{\pi_{k(l)}}(\omega)} \gamma_l + \beta_l
\end{equation}
for some fixed $\epsilon > 0$ and $\gamma, \beta \in \R^{d_{in}}$ learnable parameters.
In particular, $\mathcal{N}$ computes the average mean and standard deviation of each scalar field, or channel, of $\omega$ grouped according to the partition $\pi_1,\dots,\pi_n$  and then normalizes each accordingly. It is common practice, which we follow here, to pick a particular partition of the channels and use it to define every group normalization operator throughout the architecture. Furthermore note that $\mathcal{N}$ could have similarly been defined with $d_{out}$ instead of $d_{in}$ and, indeed, we have $\mathcal{N}_1 : \mathcal{U}(\Omega;\R^{d_{in}}) \to \mathcal{U}(\Omega;\R^{d_{in}})$ and $\mathcal{N}_2 : \mathcal{U}(\Omega;\R^{d_{out}}) \to \mathcal{U}(\Omega;\R^{d_{out}})$.

\subsection{Projections and Blocks}

In the following, we will assume that $\U(\Omega)$ is a separable Hilbert space for simplicity but remark that the Banach space setting follows similarly. Let $\mathcal{B} = \{\varphi_1,\varphi_2,\dots\}$ be an orthonormal basis of $\U$. We define the mapping $D : \U \to \mathbb{N} \cup \{\infty\}$ as 
\begin{equation}
    D(\omega) = | \{\langle \omega, \varphi \rangle : \langle \omega, \varphi \rangle \neq 0, \varphi \in \mathcal{B}\} |.
\end{equation}
In particular, $D$ counts the number of non-zero coefficients of $\omega$ in the basis $\mathcal{B}$. We then define the family of projection operators
\begin{align}
    P_n (\omega) &= \sum_{j=1}^n \langle \omega, \varphi_j \rangle \varphi_j, \\
    P_\infty(\omega) &= I_{\mathcal{U}},
\end{align}
for any $n \in \N$ and $I_\mathcal{U}$ denotes the identity operator on $\mathcal{U}$. For a sequence of layers $L_1,\dots,L_k$ defined by \eqref{eq:one_layer}, we define the middle block as
\begin{equation}
    M(\omega) = (P_{D(\omega)} \circ L_k \dots \circ P_{D(\omega)} \circ L_1)(\omega).
\end{equation}
We introduce projections after each layer to ensure that the same computational grid is kept after each application of a layer. In practice, data are usually given on a grid and therefore $D(\omega) < \infty$ for initial input. After an application of the kernel integral operator defined by \eqref{eq:convolution} as well as the Nemistkii operator defined by $\sigma$, frequencies at all wavenumbers may be introduced hence the projections ensure the same grid continues to be used for the computation as is normally done in practice.

Furthermore, for two sequences of layers $\hat{L}_1,\dots,\hat{L}_k$ and $\tilde{L}_1,\dots,\tilde{L}_k$ defined by \eqref{eq:one_layer}, we define an up-sampling and a down-sampling block as
\begin{align}
    B_{\text{d}} (\omega) &= (P_{\floor*{\frac{D(\omega)}{2}}} \circ \hat{L}_k \circ P_{D(\omega)} \circ \hat{L}_{k-1} \circ \dots P_{D(\omega)} \circ \hat{L}_1) (\omega) \\
    B_{\text{u}} (\omega) &= (P_{2 D(\omega)} \circ \tilde{L}_k \circ P_{D(\omega)} \circ \tilde{L}_{k-1} \circ \dots P_{D(\omega)} \circ \tilde{L}_1) (\omega).
\end{align}
Similarly to the middle block, the inner projections ensure the same computational grid is used while the last most projections either down-sample or up-sample the input function. Typically we choose $L_1,\dots,L_{k-1} : \mathcal{U}(\Omega;\R^{d_{in}}) \to \mathcal{U}(\Omega;\R^{d_{in}})$ and $L_k : \mathcal{U}(\Omega;\R^{d_{in}}) \to \mathcal{U}(\Omega;\R^{d_{out}})$. Furthermore, we use linear splines to define the basis set $\mathcal{B}$ \cite{devore1993constructive} and, in practice, pick $k=2$ for each block.

\subsection{U-shaped Architecture}
\label{u-shaped}

We are now in position to define the entire U-Net based neural operator. We assume that the input $\omega \in \mathcal{U}(\Omega; \R^{d_{in}})$ and that the output is a member of $\mathcal{U}(\Omega;\R^{d_{out}})$. Let $T \in \mathbb{N}$ with $T \geq 3$ be so that $T-1$ represents the number of down-sampling or up-sampling blocks to be applied. Let $d_1,\dots,d_T \in \mathbb{N}$ be a sequence of dimensions indicating the number of channels after the application of some operation. We define the following sequence of operations,

\begin{align}
    \label{eq:local_no}
    \begin{split}
    \omega^{(1)} &= (P_{D(\omega)} \circ \mathcal{F}_1) (\omega), \\
    \omega^{(j)} &= B_{\text{d}}^{(j-1)} \big ( \omega^{(j-1)} \big ), \qquad \qquad \qquad\:\:\:\: 2 \leq j \leq T, \\
    \omega^{(T+1)} &= M(\omega^{(T)}), \\
    \omega^{(T+j)} &= B_{\text{u}}^{(j-1)} \big ( (\omega^{(T+j-1)}, \omega^{(T-j+2)}) \big ), \quad 2 \leq j \leq T, \\
    \omega^{(2T+1)} &= \big( P_{D(\omega)} \circ \mathcal{F}_2 \big ) \big ( \omega^{(2T)} \big ),
    \end{split}
\end{align}
where $\omega^{(2T+1)}$ is the output of the neural operator. The operator $\mathcal{F}_1 : \mathcal{U}(\Omega;\R^{d_{in}}) \to \mathcal{U}(\Omega;\R^{d_{1}})$ is called the lifting operator since, usually, $d_1$ is chosen so that $d_1 > d_{in}$. Then each down-sampling block is a mapping $B_{\text{d}}^{(j-1)} : \mathcal{U}(\Omega;\R^{d_{j-1}}) \to \mathcal{U}(\Omega;\R^{d_{j}})$ for $2 \leq j \leq T$. It is common practice to choose $d_j > d_{j-1}$. The middle block is a mapping $M : \mathcal{U}(\Omega;\R^{d_{T}}) \to \mathcal{U}(\Omega;\R^{d_{T}})$ and the up-sampling blocks are mappings $B_{\text{u}}^{(j)}: \mathcal{U}(\Omega;\R^{d_{T-j+1}} \times \R^{d_{T - j + 1}}) \to \mathcal{U}(\Omega;\R^{d_{T-j}})$ for $1 \leq j \leq T-1$. In particular, they take as input the output of the previous block as well as the down-sampling corresponding of the same latent dimension. This ensures that information is not lost due to the down-sampling operation at the beginning of the architecture. The final operator $\mathcal{F}_2 : \mathcal{U}(\Omega;\R^{d_{1}}) \to \mathcal{U}(\Omega;\R^{d_{out}})$ is usually called the projection as it linearly maps onto the expected number of output channels. A pictorial representation of the architecture is shown in Figure~\ref{fig:architecture}.

\subsection{Kernel-Based Extension Operator}
\label{sec:extension_operator}
The neural operator architectures described thus far are designed to process functions defined on the entire domain $\bar\Omega$. However, the iterative Schwarz method, and PDE problems in general, require the incorporation of boundary conditions specified on the boundary manifold $\partial \Omega$. To bridge this gap, we introduce a dedicated \textit{boundary processor}, implemented as a kernel-based extension operator, $\mathcal{V}$. This operator maps a function defined on the boundary, $g: \partial \Omega \to \mathbb{R}^{d_u}$, to a function defined over the entire domain, $(\mathcal{V}g): \bar\Omega \to \mathbb{R}^{d_u}$. The resulting domain-wide representation of the boundary data is then concatenated with the primary input fields (i.e., the microstructure $a$) and fed into the main neural operator network.

The extension operator $\mathcal{V}$ is defined by the integral transform:
\begin{equation}
(\mathcal{V}g)(x) = \int_{\partial \Omega} \kappa_g(x, s) g(s) \, ds, \quad \forall x \in \Omega,
\label{eq:extension_operator-1}
\end{equation}
where $\kappa_g: \Omega \times \partial \Omega \to \mathbb{R}$ is a specifically designed boundary kernel. And for points on the boundary itself, the operator satisfies the identity condition:
\begin{equation}
(\mathcal{V}g)(x) = g(x), \quad \forall x \in \partial\Omega,
\label{eq:extension_operator-2}
\end{equation}

The boundary kernel $\kappa_g(x, s)$ is defined to smoothly interpolate the boundary data into the interior while ensuring consistency at the boundary itself:
\begin{equation}
\kappa_g(x, s) = \frac{f_d(x, s)}{\int_{\partial\Omega} f_d(x, t) \, dt}, \quad \forall x \in \Omega.
\label{eq:boundary_kernel}
\end{equation}
The function $f_d(x, s)$ models the influence of a boundary point $s$ on an interior point $x$ based on the Euclidean distance between them. For this work, we employ an inverse-square distance weighting:
\begin{equation}
f_d(x, s) = \|x - s\|_2^{-2}.
\end{equation}
This choice assigns greater weight to boundary points in closer proximity to the target point $x$, resulting in a smoothly varying field that accurately reflects the influence of the local boundary conditions. The normalization in the denominator of the kernel ensures that $\int_{\partial\Omega} \kappa_g(x, s) ds = 1$ for all interior points. This formulation ensures that the value of the extended function at any interior point $x$ is a weighted average of all values on the boundary.

This design helps to maximize the preservation and transfer of boundary information to the subsequent network layers, in contrast to simpler methods that might only consider a subset of boundary points or use simple averaging. It provides the neural operator with a rich, continuous representation of the boundary constraints, enhancing its ability to learn the correct solution manifold. While the kernel function $\kappa_g$ could, in principle, also be learned, we found empirically that this fixed, physics-informed formulation performs robustly and effectively.

With this extension operator defined, if we wish to approximate a mapping $(a, g) \mapsto u$, where $a: \Omega \to \mathbb{R}^{d_a}$ and $g: \partial \Omega \to \mathbb{R}^{d_u}$, we first compute the extended boundary function $\mathcal{V}g$. The final input to our main neural operator $\tilde{\mathcal{G}}$ is the channel-wise concatenation of the primary input and the extended boundary function:
\begin{equation}
\hat{u} = \hat{\mathcal{G}}(a, g) = \tilde{\mathcal{G}}\left( (a, \mathcal{V}g) \right).
\end{equation}

\section{Numerical Experiments}
\label{sec:numerics}
In this section, we show that it is possible to find good approximations of the global PDE using with a nueral operator surrogate model of the local map defined in \ref{u-shaped}. We focus on demonstrating this to the elliptical PDE as defined in equation~\eqref{eq:second_order_elliptic_pde}.

\subsection{Pretrained PPNO: Learning the Local PDE}
\label{section4.1}
We first pre-train a PPNO to approximate the local PDE defined on a square unit $\Omega = (0,1)^2$. We focus our attention on the the random Voronoi crystal microstructure and let $a$ be piecewise constant taking a single value on each Voronoi cell. The locations of the cell centers are randomly sampled from a uniform measure supported in $\Omega$, while the values of $a$ on each Voronoi cell area are sampled from a uniform measure supported on [0, 10]. The pre-training data is generated with a fixed number of cells that equals 50. A key goal of the pre-trained model is its ability to generalize to Dirichlet boundary conditions. To achieve this, we first construct random periodic functions: 
\begin{align} \label{eq:random_boundary}
    \tilde{g}(s) = & \sum_{n=1}^N \Bigg[\frac{1}{(n+1)^k} a_n \cos(2\pi n (s+s_0) + b_n) + \\& \frac{1}{(n+1)^k} c_n \sin(2\pi n (s+s_0) + d_n) \nonumber + e_n \Bigg], \quad s\in [0, 1),
\end{align} 
where we fix $N = 15$, $k = 2.5$, and sample $a_n, c_n \in U(0.5, 1)$, $b_n, d_n \in U(\frac{\pi}{4}, \frac{5\pi}{4}) $, $e_n \in U(-\frac{1}{4},\frac{1}{4})$ and $s_0 \in U(0, 1)$ i.i.d. from uniform measures. 

We develop an injective mapping $\psi: [0, 1] \mapsto \partial\Omega$ to map $s \in [0, 1]$ into $x = (x_0, x_1) \in \partial\Omega$, defined as
\begin{equation}
\psi(s) = (x_0, x_1) =
\begin{cases}
  (4s, 1) &  \text{if } s \in [0, \frac{1}{4}) \\
  (1,-4s+2) &  \text{if } s \in [\frac{1}{4}, \frac{1}{2}) \\
  (-4s+3, 0) &  \text{if } s \in [\frac{1}{2}, \frac{3}{4}) \\
  (0,4s-3) &  \text{if } s \in [ \frac{3}{4}, 1)
\end{cases}
\end{equation}

Each boundary condition $g(x)$ is then generated from the random periodic function as 
\begin{equation} 
g(x) = \tilde{g}(\psi^{-1}(x))
\end{equation}

An example of the random Voronoi microstructure, as well as the random boundary condition, is shown in Fig. \ref{fig:pretrain-example} (a) and (b), respectively.

\begin{figure}[h!]
    \centering
    \includegraphics[width=0.9\linewidth]{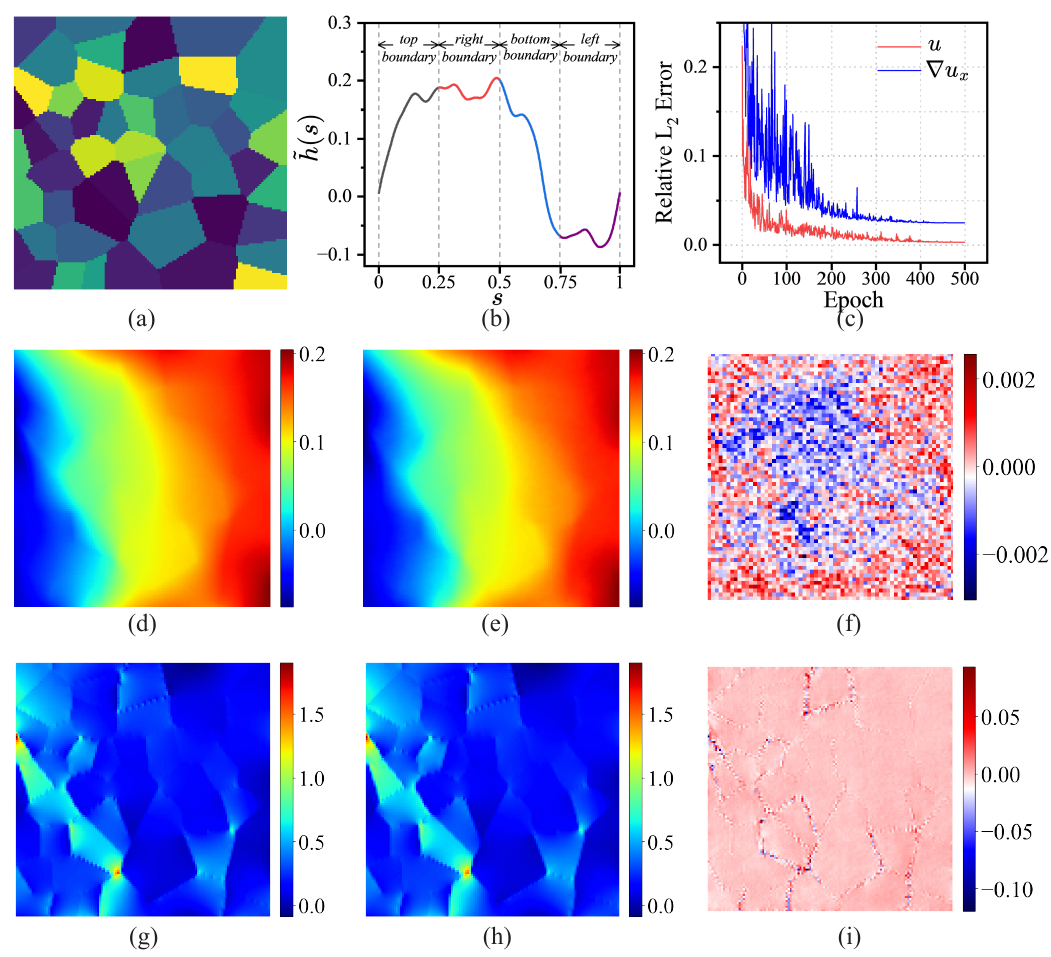}
    \caption{(a) An example of the random Voronoi microstructure. (b) An example of the periodic master function $\tilde{h}(s)$. (c) Test error versus training epochs for both $\hat{u}$ and $\nabla_x \hat{u}$. (d-f) The ground truth $u$ field(d), PPNO prediction $\hat{u}$ field(e) and error(f) for median error sample. (g-i) The ground truth(g), PPNO prediction(h) and error(i) of $\nabla_x u$ field for median error sample.}
    \label{fig:pretrain-example}
\end{figure}

Two PPNO models mapping the random Voronoi microstructure and random boundary conditions to the solution $\hat{u}$ and the gradient $\nabla \hat{u}$ are trained respectively. The models share the same architecture and for the sake of simplicity, we detail the training strategy for the PPNO model that predicts the $\hat{u}$ field, and note that the training strategy for predicting the $\nabla \hat{u}$ field remains the same.

The model is pretrained using the empirical estimate of the mean squared $L_2$ norm:
\begin{equation}
\text{Loss}(\theta) = \frac{1}{N} \sum_{n=1}^N \left( \| \hat{u}_{n} - u_{n} \|^2_{L^2}  \right),
\end{equation}
where $n$ is the sample index, $u$ is the true solution and $\hat{u}$ is the PPNO approximation of the solution parameterized by $\theta$. In the analysis, we test the pretrained model with the relative $L_2$ error: 
\begin{equation}
\text{Relative } L_2 \text{ Error} \text{ (RLS)} = \frac{1}{N} \sum_{n=1}^N \frac{\| \hat{u}_{n} - u_{n} \|^2_{L^2} }{\| u_{n} \|^2_{L^2}}.
\end{equation}

A total of 10,000 data samples are generated in the pre-training phase. The models are trained with 8000 samples generated using the Matlab finite element solver \cite{PDEToolbox} with a fixed grid resolution $129 \times 129$. 2000 data samples are used to test the trained model. Visualizations of median-error test samples for the prediction of $u$ and $\nabla_x{u}$ can be found in Fig. \ref{fig:pretrain-example}. We use Adam optimizer with CosineAnnealing scheduler with a learning rate of 5e-5 and batch size of 10. For a model trained on 10k data for 500 epochs using these hyperparameters and accelerated with an NVIDIA A100 GPU, the training time is approximately 16 hours.
\begin{table}[H]
    \centering
    \begin{tabular}{cccc}
    \hline
         Model & Nparameters & RLS\\ \hline
         PPNO & 35,281,537 & \textbf{0.2846\%} \\
         FCNN \cite{rumelhart1986learning} & 34,099,458 & 4.436\% \\
         CNN \cite{krizhevsky2012imagenet} & 37,243,265 & 4.625\% \\
         ViT \cite{dosovitskiy2020image} & 34,629,945 & 1.908\% \\
         FNO \cite{li2020fourier} & 39,537,409 & 1.057\% \\
    \hline
    \end{tabular}
    \caption{Benchmark Test}
    \label{tab:benchmarks}
\end{table}



As can be seen in Figure \ref{fig:pretrain-example}, trained PPNO models are capable of predicting the solution fields with good accuracy. The model predicting the gradient field $\nabla_x \hat{u}$ leads to a higher relative $L_2$ error (2.395\%) compared to that of the $\hat{u}$ field (0.2846\%). We observe that the error in predicting $\nabla \hat{u}$ is one order of magnitude higher on the boundary of the Voronoi cell compared to that inside the cell, which is consistent with the findings from \cite{bhattacharya2024learning}.

\textbf{Benchmarks} We benchmark the proposed PPNO model against other widely adopted models commonly employed in PDE learning tasks. Specifically, we adapt five previously published architectures: Fully Connected Neural Network (FCNN), Convolutional Neural Network (CNN), Transformer, Vision Transformer (ViT), and Fourier Neural Operator (FNO). To ensure a fair comparison, we have carefully adjusted each model so that their parameter counts are approximately equal. All models are trained and evaluated on the same dataset utilized by the pre-trained PPNO. The comparative results, including model parameters, inference time, and test loss, are presented in Table \ref{tab:benchmarks}. We observe that the PPNO demonstrates superior performance with lowest relative $L_2$  error when tested against unseen data.

\textbf{Resolution robustness} We further investigate the discretization robustness of the PPNO model. The results of this experiment are presented in Figure \ref{fig:invariant-loss}(a). By training and testing an identical PPNO model across a range of spatial resolutions, we observe that the test error remains consistently low, indicating strong robustness to changes in discretization. We emphasize that evaluating the PPNO on different resolutions is trivial in implementation by design. 

\begin{figure}[H]
    \centering
    \includegraphics[width=1.0\linewidth]{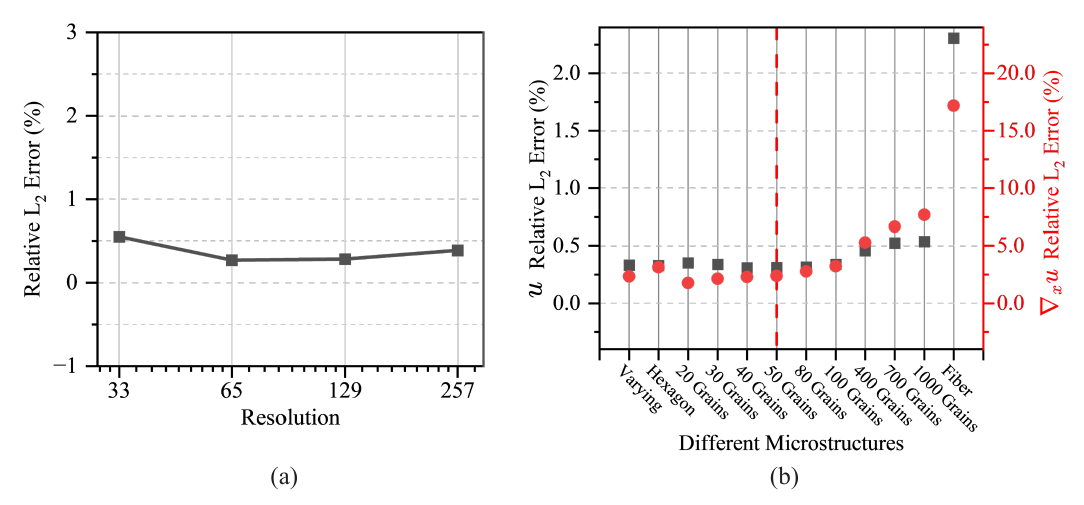}
    \caption{(a) Relative L2 test error of pre-trained PPNO model on 4 different resolutions. (b) Relative L2 error of pre-traiend PPNO model evaluated on test samples with different number of grains or microstructures for both $\hat{u}$ and $\nabla_x \hat{u}$.}.
    \label{fig:invariant-loss}
\end{figure}

\textbf{Out-of-Distribution Generalization} \label{sec:pretrain-out-of-distribution} We demonstrate that the trained PPNO model effectively generalizes to microstructure patterns significantly different from those encountered during training. Specifically, the model's performance is evaluated on four distinct classes of microstructure patterns:

\begin{enumerate}
\item Random Voronoi crystal microstructures with uniformly distributed cell centers, featuring a wide range of cell counts from 10 to 1000.
\item Graded Voronoi crystal microstructures characterized by a skewed distribution of cell centers biased toward the bottom-left corner of the domain.
\item Hexagonal microstructures approximating a hexagonal close-packed (HCP) crystal structure, where the field $a$ is piecewise constant within each hexagonal cell.
\item Fibre composite microstructures comprising a matrix embedded with circular inclusions (fibres), wherein the field $a$ assumes distinct constant values in the matrix and fibres.
\end{enumerate}

Representative visualizations selected based on median test errors from these four classes are presented in Figure \ref{fig:microstructure-invarient}, accompanied by the approximate solutions of $u$ and $\nabla u$ predicted by a PPNO model merely trained on microstructures containing 50 grains. The full visualizations of the model predictions on all other cases listed above can be found in Appendix Figure \ref{fig:ap-test-u-ux}. Quantitative comparisons using RLS errors are summarized in Figure \ref{fig:invariant-loss}(b). Remarkably, the relative $L_2$ error exhibits minimal variation despite significant changes in the grain count, spanning a range from 10 to 1000. Furthermore, the errors remain nearly constant under moderate variations in the statistical properties governing microstructure generation, as demonstrated by the skewed Voronoi and hexagonal cases. Although a modest increase in the RLS error is noted when testing the model against fibre-composite microstructures---whose statistical characteristics differ substantially from those of Voronoi crystals---the model maintains bounded errors and accurately predicts local concentration patterns in the gradient field between adjacent fibers.

\begin{figure}[H]
    \centering
    \includegraphics[width=0.9\linewidth]{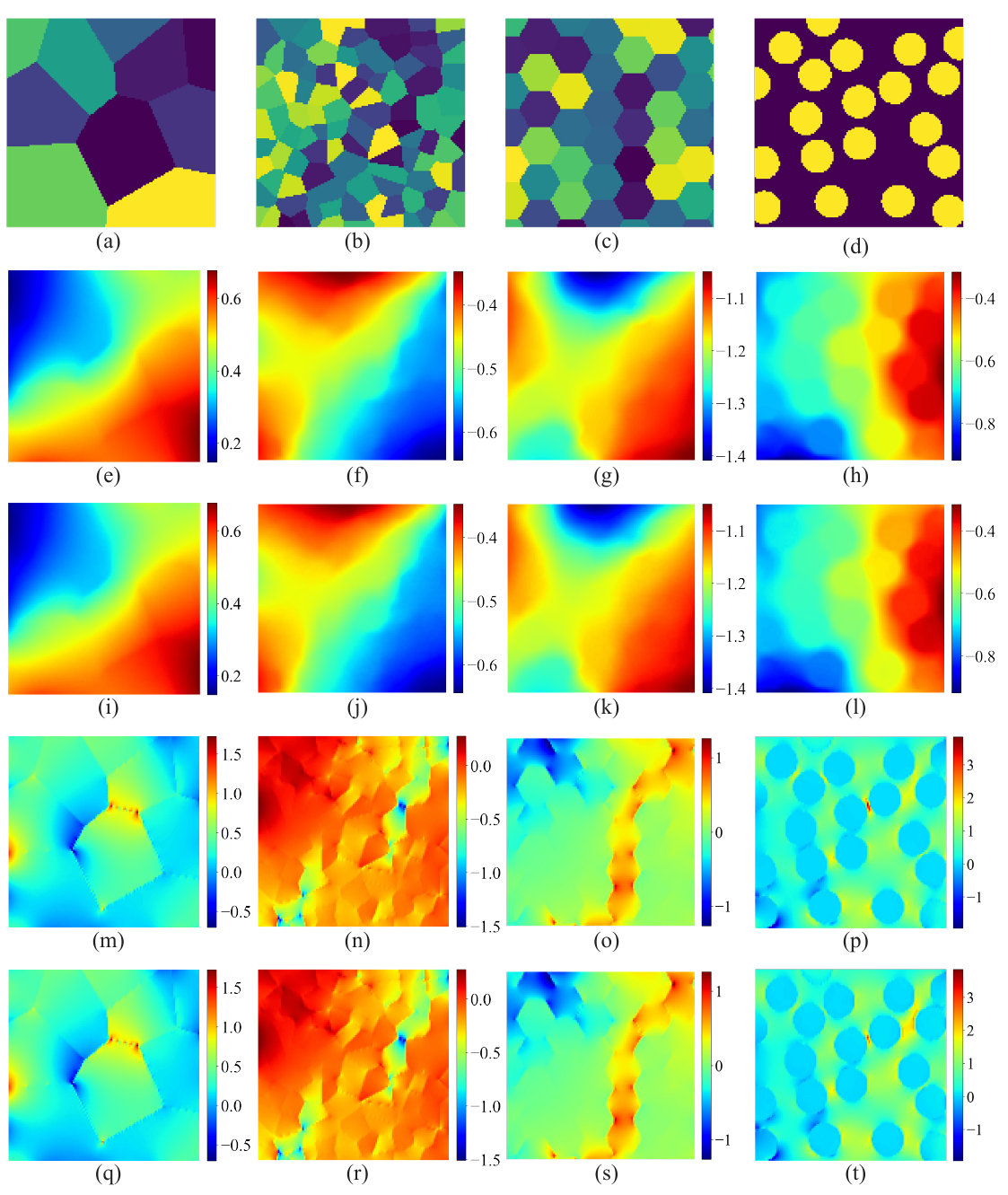}
    \caption{(a-d) Voronoi crystal microstructure of 10 random grains, 100 grains, hexagonal microstructure, and fiber composite microstructure. (e-h) Ground truth of $u$ for (a-d) with random boundary conditions. (i-l) PPNO prediction $\hat{u}$ for all microstructures. (m-p) Ground truth of $\nabla u$ for (a-d) with random boundary conditions. (q-t) PPNO prediction $\nabla \hat{u}$ for all microstructures.}
    \label{fig:microstructure-invarient}
\end{figure}

\subsection{Learning-based Domain Decomposition}
\label{section_4.2} 
Having established the PPNO, we now demonstrate its application within the learning-based domain decomposition framework, following the additive Schwarz procedure detailed in Algorithm~\ref{alg:domain_decomposition}. We evaluate the method on a series of test cases with increasing geometric and microstructural complexity to assess its accuracy, efficiency, and robustness.

\textbf{Large-Scale Square Domain with a Random Voronoi Microstructure}

We first assess the scalability of our approach by applying it to an elliptical PDE defined on a large square domain that is 50 times larger than the the domain employed in the training data set, as illustrated in Figure \ref{fig:square_large}(b). Specifically, we generate 2583 grains by randomly sampling cell centers from a uniform distribution supported on the domain $\Omega = [0,kL_0]^2$, where $L_0$ represents the length of the square domain from the pre-training dataset, and $k = 7.1875$ is a fixed scaling factor. 

For the solution process, the global domain is partitioned into a $10 \times 10$ grid of identical square subdomains with a 31.25\% overlap ratio, as depicted in Figure \ref{fig:square_large}(c). Each subdomain thus contains approximately 50 grains, consistent with the scale of the pre-training data. Random Dirichlet boundary conditions, generated according to Eq.~\eqref{eq:random_boundary} with parameters specified in \ref{boundary_params}, are applied to the exterior edges (Figure \ref{fig:square_large}(c)).

Figure \ref{fig:solution-square} presents the results of this simulation. The learning-based solution demonstrates excellent agreement with the ground truth for both the $\hat{u}$ field and its gradient $\nabla_x\hat{u}$. The convergence plots (Figure \ref{fig:solution-square}(c,f)) show that the successive error rapidly diminishes, although it requires approximately 100 iterations to reach the tolerance threshold. Upon convergence, the final relative $L_2$ error against the ground truth solution stabilizes at 0.645\% for the $\hat{u}$ field and 5.99\% for the $\nabla_x\hat{u}$ field. Notably, this final accuracy error is of the same order of magnitude as the generalization error of the pre-trained PPNO on single-domain test cases. This indicates that the domain decomposition framework does not introduce significant additional error, and the overall accuracy is primarily limited by the performance of the underlying neural operator.

\begin{figure}[H]
    \centering
    \includegraphics[width=1.0\linewidth]{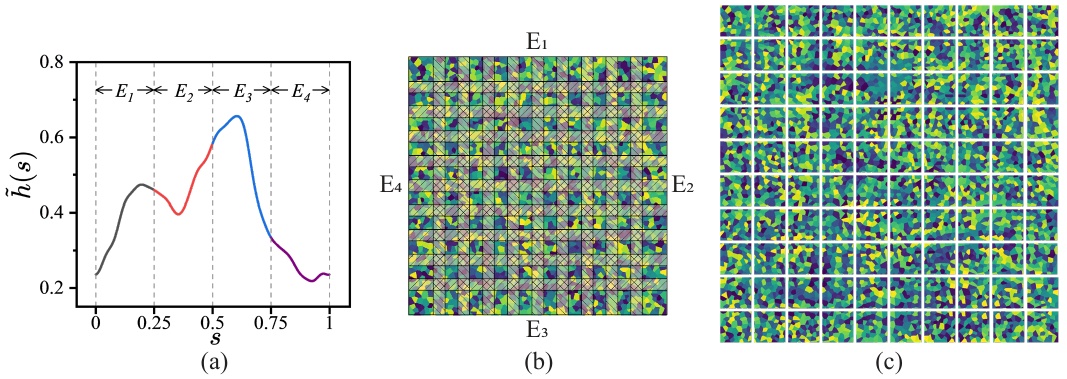}
    \caption{(a) The random Dirichlet boundary conditions applied to the 4 edges of the scaled square domain. (b) The square domain with Voronoi crystal microstructure decomposed into $10 \times 10$ subdomains, where the shaded area represents the overlapping of subdomains. (c) 100 subdomains each containing 50 grains after decomposition.}
    \label{fig:square_large}
\end{figure}

\begin{figure}[H]
    \centering
    \includegraphics[width=1.0\linewidth]{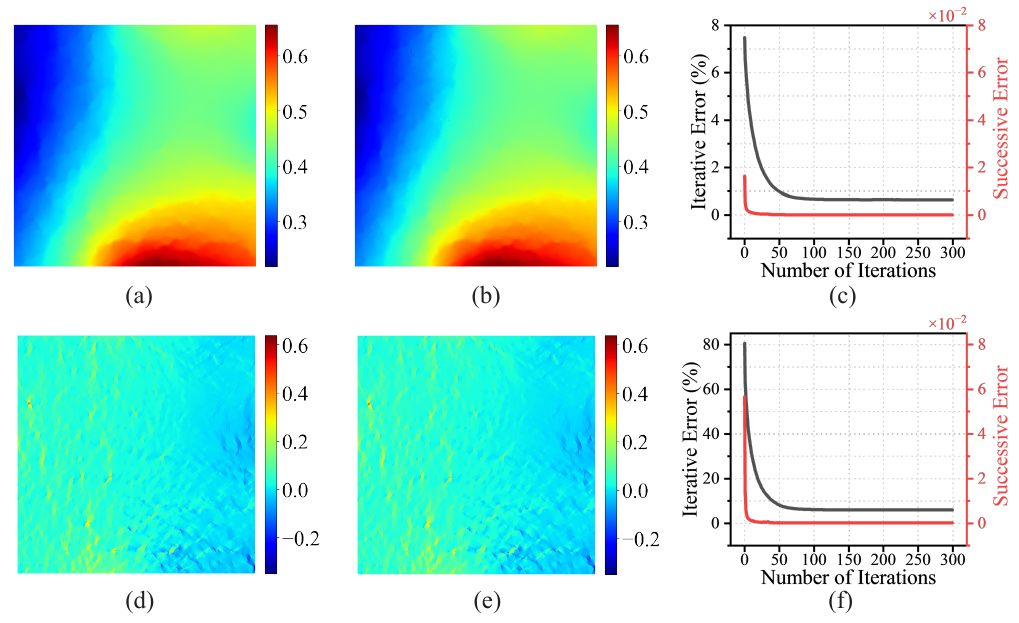}
    \caption{(a-b) The truth solution $u$ and learning-based solution $\hat{u}$ field for the $10 \times 10$ rectangular domain, with boundary conditions prescribed in Fig.\ref{fig:square_large}c applied. (c) The iterative (black line) and successive (red line) error versus the number of iterations of the $\hat{u}$ field. (d-e) The truth solution $\nabla_x u$ and learning-based solution $\nabla_x\hat{u}$ field for the same rectangular domain with the same boundary conditions applied. (g) The iterative (black line) and successive (red line) error versus the number of iterations of the $\nabla_x\hat{u}$ field.}
    \label{fig:solution-square}
\end{figure}

\textbf{L-Shaped Domain with a High-Density Voronoi Microstructure}

Next, to evaluate the method's performance on a non-convex geometry, we consider an L-shaped domain. This domain is partitioned into 12 overlapping subdomains (Figure \ref{fig:Lshape}(b-c)). The microstructure in this case is a high-density random Voronoi tessellation, with approximately 500 grains per subdomain, representing an out-of-distribution test case as described in Section \ref{sec:pretrain-out-of-distribution}. Random Dirichlet boundary conditions are applied to the six exterior edges of the domain (Figure \ref{fig:Lshape}(a)).

The results, shown in Figure \ref{fig:Lshape-solution}, again confirm the method's robustness. The learning-based solution accurately captures the complex field patterns induced by the non-convex geometry. The iterative process exhibits rapid convergence, with the successive error decaying to the tolerance threshold within 20 iterations (Figure \ref{fig:Lshape-solution}(c,g)). The final relative $L_2$ error against the ground truth is 0.388\% for the displacement field $\hat{u}$ and 8.10\% for its gradient $\nabla_x\hat{u}$. This successful application demonstrates the framework's ability to handle complex geometries and out-of-distribution microstructures effectively.


\begin{figure}[H]
    \centering
    \includegraphics[width=1.0\linewidth]{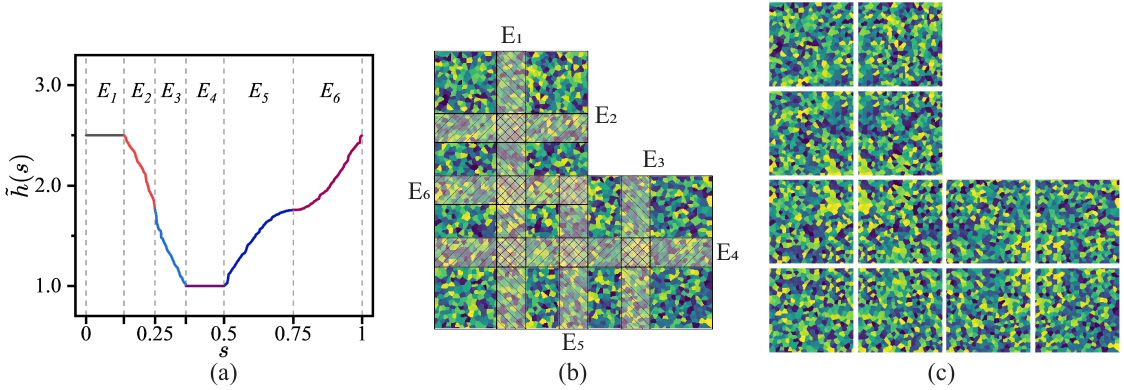}
    \caption{(a) The random Dirichlet boundary conditions applied to the 6 edges of the L-shaped domain. (b) The L-shaped domain with Voronoi crystal microstructure decomposed into 12 subdomains, where the shaded area represents the overlapping of subdomains. (c) 12 subdomains each containing 500 grains after decomposition.}
    \label{fig:Lshape}
\end{figure}

\begin{figure}[H]
    \centering
    \includegraphics[width=1.0\linewidth]{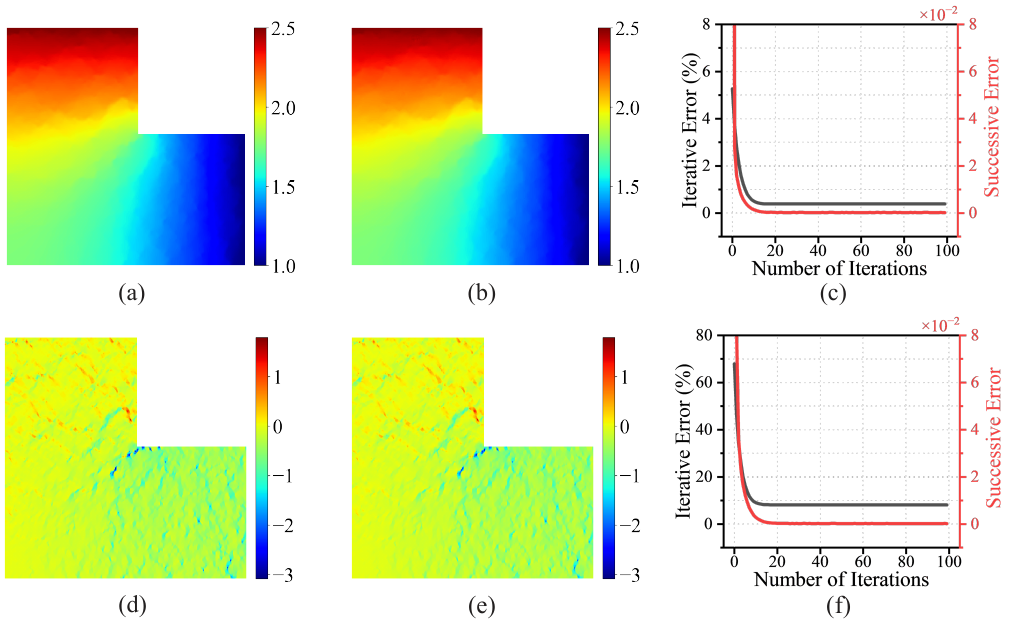}
    \caption{(a-b) The truth solution $u$ and learning-based solution $\hat{u}$ field for the L-shaped domain, with boundary conditions prescribed in Fig.\ref{fig:Lshape}c applied. (c) The iterative (black line) and successive (red line) error versus the number of iterations of the $\hat{u}$ field. (e-f) The truth solution $\nabla_x u$ and learning based solution $\nabla_x \hat{u}$ field for the same L-shaped domain with the same boundary conditions applied. (g) The iterative (black line) and successive (red line) error versus the number of iterations of the $\nabla_x\hat{u}$ field.}
    \label{fig:Lshape-solution}
\end{figure}

\textbf{I-Shaped Domain with a Graded Voronoi Microstructure}

As a final and most challenging test case, we examine an I-shaped domain, featuring a spatially graded Voronoi microstructure. The average grain size in this material gradually increases from the left to approximately 65 times at the right end of the domain (Figure \ref{fig:TTS}(b)). This problem combines a complex, non-convex geometry with significant material heterogeneity. The specimen is constructed on the domain $\Omega=[0,k_2L_0]\times[0,k_3L_0]$, where $k_2=7.8217, k_3=2.3643$, with 720 grains in total. The domain is decomposed into 23 overlapping subdomains, and prescribed Dirichlet boundary conditions are applied to its twelve exterior edges (Figure \ref{fig:TTS}(a,c)).

Despite the increased complexity, the learning-based framework produces a highly accurate solution, as shown in Figure \ref{fig:TTS-solution}. There is excellent visual agreement between the ground truth and the predicted fields for both $\hat{u}$ and $\nabla_x\hat{u}$. The final converged solution achieves a relative $L_2$ error of just 0.0392\% for the $\hat{u}$ field and 3.89\% for the $\nabla_x\hat{u}$ field. Similar to the L-shaped domain case, the algorithm demonstrates rapid convergence, requiring only 20 iterations. The exceptional performance on this challenging problem underscores the robustness of the PPNO and its effective integration within the additive Schwarz framework, capable of handling combined geometric and material complexities.



\begin{figure}[H]
    \centering
    \includegraphics[width=1.0\linewidth]{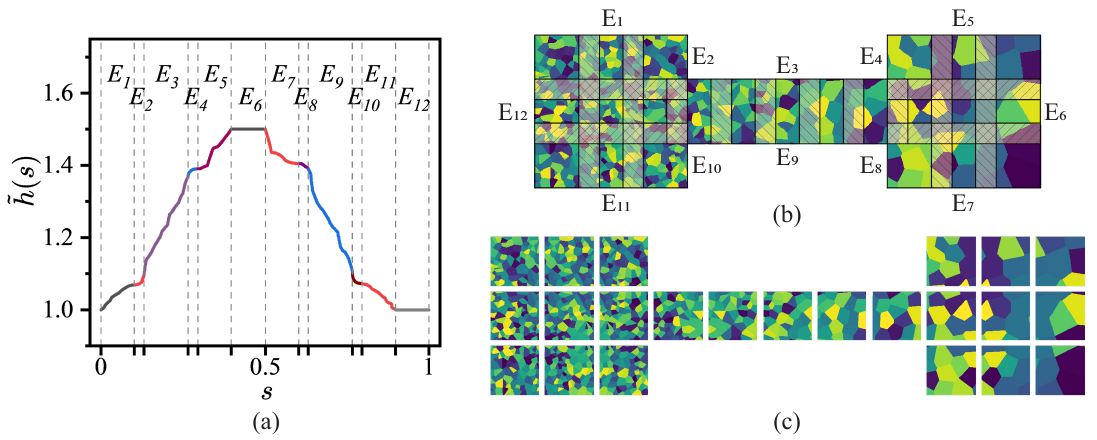}
    \caption{(a) The Dirichlet boundary conditions applied to the 12 edges of the I-shaped domain. (b) The I-shaped domain with Voronoi crystal microstructure decomposed into 23 subdomains, where the shaded area represents the overlapping of subdomains. (c) 23 subdomains containing graded number of grains after decomposition.}
    \label{fig:TTS}
\end{figure}

\begin{figure}[H]
    \centering
    \includegraphics[width=1.0\linewidth]{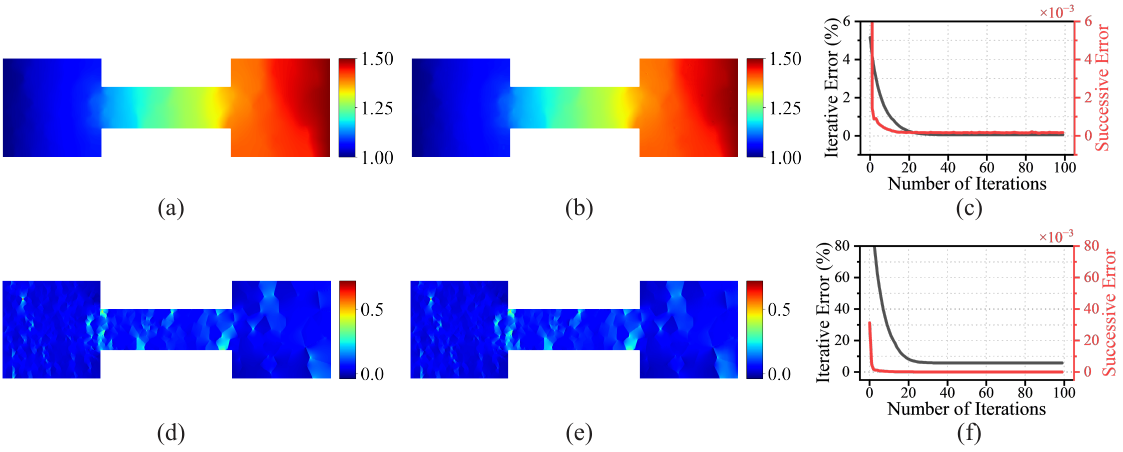}
    \caption{(a-b) The truth solution $u$ and learning-based solution $\hat{u}$ field for the I-shaped domain, with boundary conditions prescribed in Fig.\ref{fig:TTS}c applied. (c) The iterative (black line) and successive (red line) error versus the number of iterations of the $\hat{u}$ field. (d-e) The truth solution $\nabla_x u$ and learning-based solution $\nabla_x \hat{u}$ field for the same domain with the same boundary conditions applied. (g) The iterative (black line) and successive (red line) error versus the number of iterations of the $\nabla_x\hat{u}$ field.}
    \label{fig:TTS-solution}
\end{figure}

To further demonstrate the robustness of the combination of PPNO and the learning-based DDM, in \ref{sec:rectangular_domains}, we apply it on rectangular domains with out-of-distribution microstructure patterns, confirming successful convergence and strong performance beyond the training distribution. A more detailed study about the algorithm’s robustness and invariance is presented in \ref{sec:algorithmic_robust}, demonstrating stable convergence and negligible sensitivity to variations in microstructure topology, boundary conditions, and initial conditions.

\section{Concluding remarks}
\label{part6}

In this work, we have introduced a learning-based domain decomposition method, synergistically combined with the proposed physics-pretrained Neural Operator (PPNO), to efficiently solve parameterized PDEs over large and geometrically complex domains. Our primary contribution is a comprehensive methodology that leverages a neural operator that is trained only once on a simple canonical domain, as an accelerated solver within a parallel additive Schwarz iterative scheme.

We have provided a rigorous theoretical foundation for this approach, establishing an approximation theorem for the learning-based domain decomposition algorithm and proving the existence of neural operator surrogate models that can be applied to a variety of PDE-based problems. The developed PPNO architecture is shown to outperform state-of-the-art models, including the Fourier Neural Operator and Vision Transformer, for this problem class. Through extensive numerical experiments, we have validated the theoretical framework, demonstrating that the proposed method achieves high accuracy and rapid convergence. Notably, the framework maintains its robustness even when applied to domains with complex geometries (e.g., L-shaped and I-shaped) and on microstructures that are significantly different from those encountered during pre-training.

While our current implementation utilizes square subdomains congruent to the pre-training domain for simplicity, our theoretical results hold for general Lipschitz domains. This points toward several promising avenues for future research. A natural extension involves applying the method to unstructured decompositions with arbitrarily shaped subdomains, which could be achieved through appropriate geometric mappings, leveraging the length-scale invariance inherent to the elliptic PDEs considered.

Furthermore, the framework could be extended to tackle more complex, nonlinear problems. This includes time-dependent phenomena and advanced material models, such as those found in hyperelasticity and plasticity, which would require adapting the operator to handle history-dependent constitutive relations. Another exciting direction is the development of multi-physics capabilities, where different governing equations are solved on different subdomains, building upon preliminary explorations in the literature. Finally, integrating the proposed method with multi-scale homogenization techniques or employing adaptive, non-uniform domain decomposition strategies could further enhance its efficiency and applicability to a broader range of challenging problems across science and engineering.

In summary, the learning-based domain decomposition method offers a powerful and flexible paradigm for high-performance scientific computing. Our work provides both the theoretical justification and practical demonstration, paving the way for accelerated and accurate simulations of complex physical systems.

\section*{Acknowledgment}
The work was funded by EPSRC grant EP/W524633/1. The authors are grateful to Dr. Yuxuan Chen for helpful discussions. NBK is grateful to the NVIDIA Corporation for support through full-time employment. 
 
\bibliographystyle{elsarticle-num} 
\bibliography{references}

\appendix
\section{Approximation of the Schwarz Method}
\label{sec:approximation_of_the_schwarz_method}

\begin{proof}[Proof of Lemma~\ref{lemma:uniform_locally_lip}]
    Let $\tilde{L} > 0$ denote the Lipschitz constant of the map $(a,g) \mapsto u$ on $K$. Since $K$ is compact, we can find elements $(a_1,g_1),\dots,(a_r,g_r) \in K$, for some $r = r(\epsilon) \in \N$, such that, for any $(a,g) \in K$, there exists $j \in \{1,\dots,r\}$ such that
    \[\|a - a_j\|_\A + \|g-g_j\|_\U < \frac{\epsilon}{2(L + \tilde{L})}.\]
    We can therefore find numbers $N_1,\dots,N_r \in \N$ such that
    \[\|\G^{(n)}(a_j,g_j) - u(a_j,g_j)\|_\U < \frac{\epsilon}{2}, \qquad \forall \: n \geq N_j.\]
    Let $N = \max \{N_1,\dots,N_r\}$. The for any $(a,g) \in K$ and $n \geq N$, we have
    \begin{align*}
        \|\G^{(n)}(a,g) - u(a,g)\|_\U &\leq \|\G^{(n)}(a,g) - \G^{(n)}(a_j,g_j)\|_\U + \|\G^{(n)}(a_j,g_j) - u(a,g)\|_\U \\
        &\leq L \big ( \|a_1 - a_2\|_\A + \|g_1 - g_2\|_\U \big ) \\ 
        &+ \|\G^{(n)}(a_j,g_j) - u(a_j,g_j)\|_U + \|u(a_j,g_j) - u(a,g)\|_\U \\
        &< \frac{\epsilon}{2} + \big ( L+\tilde{L} \big ) \big ( \|a_1 - a_2\|_\A + \|g_1 - g_2\|_\U \big ) \\
        &< \epsilon
    \end{align*}
    as desired.
\end{proof}

\begin{proof}[Proof of Theorem~\ref{thm:main_approximation}]
    \textbf{Step 1.} Since $K$ satisfies the uniform additive Schwarz property, there exists $n \in \N$ such that
    \[\sup_{(a,g) \in K} \|\G^{(n)}(a,g) - u(a,g)\|_{C^s} < \frac{\epsilon}{2},\]
    hence it is enough to show the existence of an additive Schwarz neural operator $\hat{\G}$ such that
    \[\sup_{(a,g) \in K} \|\hat{\G}^{(n)}(a,g) - \G^{(n)}(a,g)\|_{C^s} < \frac{\epsilon}{2}.\]
    Since $a$ is kept fixed within the composition $\G^{(n)}(a,g)$, this result follows exactly as in the proof of Theorem~\ref{thm:composition_cs_cs}. Therefore, we need only show that $\G(a,\cdot)$ is locally Lipschitz continuous for every $a \in \U$ and that there exists an additive Schwarz neural operator $\hat{\G}$ such that
    \[\sup_{(a,g) \in K} \|\hat{\G}(a,g) - \G(a,g)\|_{C^s} < \epsilon.\]
    
    \textbf{Step 2.} To that end, let $a \in \U$ and $H \subset \U$ be any compact set. Let $L_1$ denote the Lipschitz constant of $\G_1$ on $H$. Furthermore, note that the set
    \[\{\G_1(a,g) \varphi + g \phi : g \in H\}\]
    is compact by continuity of $\G_1$. Let $L_2$ denote the Lipschitz constant of $\G_2(a,\cdot)$ on this set. Then for any $g_1,g_2 \in H$, we have
    \begin{align*}
        \|\G(a,g_1) - \G(a,g_2)\|_{C^s} &\leq \|\G_1(a,g_1) \varphi - \G_1 (a,g_2) \varphi \|_{C^s} \\
        &+ \|\G_2 \big( a,\G_1(a,g_1) \varphi + g_1 \phi \big ) \phi - \G_2 \big( a,\G_1(a,g_2) \varphi + g_2 \phi \big ) \phi \|_{C^s} \\
        &\leq \big ( L_1 \|\varphi\|_{C^s}  + L_2 \|\phi\|_{C^s} (L_1 \|\varphi\|_{C^s} + \|\phi\|_{C^s}) \big ) \|g_1 - g_2\|_{C^s} 
    \end{align*}
    which implies that $\G(a,\cdot)$ is locally Lipschitz continuous.
    
    \textbf{Step 3.} Consider the set 
    \[\tilde{K} =  \big \{ \big ( a, \G_1 (a,g) \varphi + g \phi \big ) : (a,g) \in K \big \}\]
    which is compact by continuity of $\G_1$. By universal approximation \cite{kovachki2023neural, lanthaler2023nonlocality}, we can find a neural operator $\hat{\G}_2$ such that
    \[\sup_{(a,g) \in \tilde{K}} \|\G_2(a,g) - \hat{\G}_2(a,g)\|_{C^s} < \frac{\epsilon}{3\|\phi\|_{C^s}}.\]
    Perusal of the universal approximation theorem shows that $\hat{\G}_2$ can be constructed so that it is globally Lipschitz. Let $\hat{L}_2 = \hat{L}_2(\epsilon)$ denote its Lipschitz constant.  We again use universal approximation to find a neural operator $\hat{\G}_1$ such that
    \[\sup_{(a,g) \in K} \|\G_1(a,g) - \hat{\G}_1(a,g) \|_{C^s} < \frac{\epsilon}{3\hat{L}_2\|\varphi\|_{C^s} \|\phi\|_{C^s}}.\]
    Without loss of generality, we may assume $\hat{L}_2, \|\varphi\|_{C^s}, \|\phi\|_{C^s} \geq 1$. It follows that
    \begin{align*}
        \sup_{(a,g) \in K} \|\G(a,g) - \hat{\G}(a,g)\|_{C^s} &\leq \sup_{(a,g) \in K} \|\G_1(a,g) \varphi - \hat{\G}_1(a,g) \varphi \|_{C^s} \\
        &+ \sup_{(a,g) \in K} \|\G_2 \big( a, \G_1(a,g) \varphi + g \phi \big ) \phi - \hat{\G}_2 \big ( a, \hat{\G}_1 (a,g) \varphi + g \phi \big ) \phi \|_{C^s} \\
        &< \frac{\epsilon}{3} + \sup_{(a,g) \in \tilde{K}} \|\phi\|_{C^s} \|\G_2(a,g) - \hat{\G}_2(a,g)\|_{C^s} \\
        &+ \sup_{(a,g) \in K} \|\hat{\G}_2 \big ( a, \hat{\G}_1(a,g)\varphi + g \phi \big ) \phi - \hat{\G}_2 \big ( a, \G_1 (a,g) \varphi + g \phi \big ) \phi \|_{C^s}  \\
        &< \frac{2 \epsilon}{3} + \hat{L}_2 \|\phi\|_{C^s} \|\varphi\|_{C^s} \sup_{(a,g) \in K} \|\G_1(a,g) - \hat{\G}_1(a,g) \|_{C^s} \\
        &< \epsilon
    \end{align*}
    as desired.
\end{proof}

\section{Composition of Operators}
\label{sec:compisiton_of_operators}

We consider a bounded Lipschitz domain $\Omega \subset \R^d$ and a locally Lipschitz continuous operator $\G : C^s (\bar{\Omega}) \to C^r (\bar{\Omega})$ where $s,r \in \N_0$ and $r \geq s$. For any \(n \in \N\), we use the notation
\[\G^{(n)}(a) = \underbrace{(\G \circ \cdots \circ \G)}_{n \text{ times}}(a), \qquad a \in C^s(\bar{\Omega}),\]
for the $n$-fold composition.

We begin by approximating $\G^{(n)}$ in the case when the output space of $\G$ is smoother than the input space. This allows us to turn boundedness 
into compactness and construct a large enough input set so that the approximation always remains within it even after $n$ compositions. Details are given below.

\begin{lemma}
    \label{lemma:compose_r_greater_s}
    Let $\G: C^s (\bar{\Omega}) \to C^r (\bar{\Omega})$ be locally Lipschitz continuous for some $r > s \geq0$. For any compact set $K \subset C^s (\bar{\Omega})$, $n \in \N$, and $\epsilon > 0$, there exists a neural operator $\hat{\G} : C^s (\bar{\Omega}) \to C^r (\bar{\Omega})$ such that
    \[\sup_{a \in K} \|\G^{(n)}(a) - \hat{\G}^{(n)}(a) \|_{C^r} < \epsilon.\]
\end{lemma}
\begin{proof}
    Since $\G(K)$ is compact in $C^r (\bar{\Omega})$, there exists a constant $M_1 > 0$ such that
    \[\sup_{a \in K} \|\G(a)\|_{C^r} \leq M_1.\]
    Define the set
    \[K_1  = \{a \in C^s(\bar{\Omega}) : \|a\|_{C^r} \leq 2M_1\}.\]
    Since $r > s$, $K_1$ is compact in $C^s(\bar{\Omega})$. Therefore $\G(K_1)$ is compact in $C^r (\bar{\Omega})$, so there exists a constant $M_2 > 0$ such that 
    \[\sup_{a \in K_1} \|\G(a)\|_{C^r} \leq M_2.\]
    Define the set 
    \[K_2 = \{a \in C^s(\bar{\Omega}) : \|a\|_{C^r} \leq 2 M_2\}\]
    which is compact in $C^s (\bar{\Omega})$. Iterating on this process, we define the set
    \[K_l = \{a \in C^s (\bar{\Omega}) : \|a\|_{C^r} \leq 2 M_l\}\] where $M_l > 0$ is such that
    \[\sup_{a \in K_{l-1}} \|\G(a)\|_{C^r} \leq M_l\]
    for any $l \geq 2.$ Define the set 
    \[\hat{K} = K \cup \bigcup_{l=1}^n K_l\]
    which is compact in $C^s (\bar{\Omega})$. Let $\hat{L}$ denote the Lipschitz constant of $\G$ on $\hat{K}$
    and define
    \[L = \sum_{l=0}^{n-1} \hat{L}^l.\]
    By the universal approximation theorem for neural operators, there exists a neural operator $\hat{\G} : C^s (\bar{\Omega}) \to C^r (\bar{\Omega})$
    such that 
    \[\sup_{a \in \hat{K}} \|\G(a) - \hat{\G}(a) \|_{C^r} < \frac{\epsilon}{L}.\]
    Without loss of generality, we may assume $M_l \geq \epsilon$ for every $l = 1,\dots,n$ and that $L \geq 1$. Notice that
    \[\sup_{a \in K} \|\hat{\G}(a)\|_{C^r} \leq \sup_{a \in K} \|\hat{\G}(a) - \G(a)\|_{C^r} + \|\G(a)\|_{C^r} \leq \epsilon + M_1 \leq 2M_1.\]
    Therefore,
    \[\sup_{a \in K} \|\hat{\G}(\hat{\G}(a))\|_{C^r} \leq \sup_{a \in K_1} \|\hat{\G}(a)\|_{C^r} \leq 2 M_2.\]
    It follows by induction that
    \[\sup_{a \in K} \|\hat{\G}^{(l)}(a)\|_{C^r} \leq 2 M_l,\]
    and, in particular, that $\hat{\G}^{(l)}(a) \in \hat{K}$ for every $l=1,\dots,n$ and $a \in K$.
    We therefore find that
    \begin{align*}
        \sup_{a \in K} \|\G^{(n)}(a) - \hat{\G}^{(n)}(a)\|_{C^r} &= \sup_{a \in K} \|\G (\G^{(n-1)}(a)) - \hat{\G}(\hat{\G}^{(n-1)}(a)) \|_{C^r} \\
        &\leq \sup_{a \in K} \|\G (\G^{(n-1)}(a)) - \G(\hat{\G}^{(n-1)}(a))\|_{C^r} \\
         &+ \sup_{a \in K} \|\G(\hat{\G}^{(n-1)}(a)) - \hat{\G}(\hat{\G}^{(n-1)}(a))\|_{C^r} \\
        &< \hat{L} \sup_{a \in K} \|\G^{(n-1)}(a) - \hat{\G}^{(n-1)}(a)\|_{C^r} + \frac{\epsilon}{L} \\
        &< \dots \\
        &< \hat{L}^{n-1} \sup_{a \in K} \|\G(a) - \hat{\G}(a)\|_{C^r} + \frac{\epsilon}{L} \sum_{l=0}^{n-2} \hat{L}^l \\
        &< \frac{\epsilon}{L} \sum_{l=0}^{n-1} \hat{L}^n \\
        &= \epsilon
    \end{align*}
    as desired.
\end{proof}

The proof of Lemma~\ref{lemma:compose_r_greater_s} immediately generalizes to the follow result which we state here for completeness.

\begin{theorem}
    Let $\A$ and $\U$ be Banach spaces such that the compact embedding $\U \ssubset \A$ holds. Assume there is a universal approximation theorem
    for neural operators approximating continuous maps $\A \to \U$. Let $\G : \A \to \U$ be locally Lipschitz continuous. Then, for any compact set $K \subset \A$, $n \in \N$, and $\epsilon > 0$, there exists a neural operator $\hat{\G} : \A \to \U$ such that
    \[\sup_{a \in K} \|\G^{(n)}(a) - \hat{\G}^{(n)}(a) \|_{\U} < \epsilon.\]
\end{theorem}

We can now state and prove the main result of this section, in particular, the case $r=s$. We establish this by approximating $\G$ first with a mollified version constructed simply by convolving its output with a mollifier. We show that this approximation can be $\epsilon$-close after $n$ compositions in a way similar to Lemma~\ref{lemma:compose_r_greater_s} except by using properties of the mollification operator to gain compactness; see \cite[Appendix A]{lanthaler2023nonlocality}. Since the mollified operator always maps into a space of smoothness greater than $s$, we can then apply Lemma~\ref{lemma:compose_r_greater_s} to finish the proof. Details are given below.

\begin{theorem}
    \label{thm:composition_cs_cs}
    Let $\G: C^s (\bar{\Omega}) \to C^s (\bar{\Omega})$ be locally Lipschitz continuous for some $s \geq 0$. For any compact set $K \subset C^s (\bar{\Omega})$, $n \in \N$, and $\epsilon > 0$, there exists a neural operator $\hat{\G} : C^s (\bar{\Omega}) \to C^{s+1} (\bar{\Omega})$ such that
    \[\sup_{a \in K} \|\G^{(n)}(a) - \hat{\G}^{(n)}(a) \|_{C^s} < \epsilon.\]
\end{theorem}
\begin{proof}
    \textbf{Step 1.} We will first construct a mollified version of $\G$. For this, we will use the family of operators $\M_\delta : L^1(\Omega) \to C^\infty (\bar{\Omega})$ constructed in \cite[Lemma A.2]{lanthaler2023nonlocality}. Fix $\delta > 0$ and define the sets
    \begin{align*}
        K_{\delta,1} &= \bigcup_{0 \leq \delta' \leq \delta} \{\M_{\delta'} \G(a) : a \in K\}, \\
        K_{\delta,l} &= \bigcup_{0 \leq \delta' \leq \delta} \{\M_{\delta'} \G(a) : a \in K_{\delta,l-1}\}, \quad l\geq 2.
    \end{align*}
    By continuity of $\G$ and \cite[Lemma A.4]{lanthaler2023nonlocality}, $K_{\delta, l}$ is compact in $C^s(\bar{\Omega})$ for each $l \geq 1$. Therefore the set
    \[K_\delta = K \cup \bigcup_{l=1}^n K_{\delta,l}\]
    is also compact in $C^s (\bar{\Omega})$. Let $L_\delta$ denote the Lipschitz constant of $\G$ on $K_\delta$ and define 
    \[L = \sum_{l=0}^{n-1} L_\delta^l.\]
    By \cite[Lemma A.2]{lanthaler2023nonlocality}, we can find some $0 < \delta_0 \leq \delta$ such that
    \[\sup_{a \in K_\delta} \|\G(a) - \M_{\delta_0} \G (a) \|_{C^s} < \frac{\epsilon}{2L}. \]
    Therefore, we have that
    \begin{align}
        \label{eq:molly_approx_truth_c0}
        \begin{split}
            \sup_{a \in K} \|\G^{(n)}(a) - (\M_{\delta_0} \G)^{(n)}(a) \|_{C^s} &= \sup_{a \in K} \|\G(\G^{(n-1)}(a)) - \M_{\delta_0} \G ( (\M_{\delta_0} \G)^{(n-1)}(a)) \|_{C^s} \\
            &\leq \sup_{a \in K} \|\G (\G^{(n-1)}(a)) - \G ( (\M_{\delta_0} \G)^{(n-1)}(a)) \|_{C^s} \\
            &+ \sup_{a \in K} \|\G ( (\M_{\delta_0} \G)^{(n-1)}(a)) - \M_{\delta_0} \G ( (\M_{\delta_0} \G)^{(n-1)}(a))\|_{C^s} \\
            &< L_\delta \sup_{a \in K} \|\G^{(n-1)}(a) - (\M_{\delta_0} \G)^{(n-1)}(a) \|_{C^s} + \frac{\epsilon}{2L} \\
            &\leq ... \\
            &\leq L^{N-1}_\delta \sup_{a \in K} \|\G(a) - \M_{\delta_0} \G (a) \|_{C^s} + \frac{\epsilon}{2L} \sum_{l=0}^{n-2} L_\delta^j  \\
            &< \frac{\epsilon}{2L} \sum_{l=0}^{n-1} L_\delta^l \\
            &= \frac{\epsilon}{2}.
        \end{split}
    \end{align}

    \textbf{Step 2.} Let $Q \subset C^s (\bar{\Omega})$ be any compact set. Notice that, for any $v,w \in Q$,
    \[\|\M_{\delta_0} \G (v) - \M_{\delta_0} \G (w) \|_{C^{s+1}} \leq \|\M_{\delta_0}\|_{C^{s} \to C^{s+1}} \|\G(v) - \G(w)\|_{C^s} \leq L_Q \|\M_{\delta_0}\|_{C^{s} \to C^{s+1}} \|v - w\|_{C^s} \]
    where $L_Q$ is the Lipschitz constant of $\G$ on $Q$ and $\|\M_{\delta_0}\|_{C^{s} \to C^{s+1}} < \infty$ by continuity \cite[Lemma A.2]{lanthaler2023nonlocality}. In particular, $\M_{\delta_0} \G : C^s (\bar{\Omega}) \to C^{s+1} (\bar{\Omega})$ is locally Lipschitz. By Lemma~\ref{lemma:compose_r_greater_s}, we can find a neural operator $\hat{\G}: C^s (\bar{\Omega}) \to C^{s+1} (\bar{\Omega})$ such that
    \[\sup_{a \in K} \|(\M_{\delta_0} \G)^{(n)} (a) - \hat{\G}^{(n)}(a)\|_{C^{s+1}} < \frac{\epsilon}{2}.\]
    Therefore, using \eqref{eq:molly_approx_truth_c0}, we find that
    \begin{align*}
        \sup_{a \in K} \|\G^{(n)}(a) - \hat{\G}^{(n)} (a)\|_{C^s} &\leq \sup_{a \in K} \|\G^{(n)}(a) - (\M_{\delta_0} \G)^{(n)}(a) \|_{C^s} + \|(\M_{\delta_0} \G)^{(n)} (a) - \hat{\G}^{(n)}(a)\|_{C^{s}} \\
        &< \frac{\epsilon}{2} + \sup_{a \in K} \|(\M_{\delta_0} \G)^{(n)} (a) - \hat{\G}^{(n)}(a)\|_{C^{s+1}} \\
        &< \epsilon
    \end{align*}
    as desired.
\end{proof}

\section{Square domains with out-of-distribution microstructure patterns}
\label{sec:squre_domains}
The trained PPNO exhibits generalization capabilities when evaluated on microstructures unseen during training. The model is trained exclusively on a dataset of square domains, each containing 50 grains. Subsequently, the model is tested on out-of-distribution (OOD) cases featuring two key variations: (i) domains with grain counts ranging widely from 10 to 1,000, and (ii) domains with non-uniformly distributed grain sizes. Figure \ref{fig:ap-test-u-ux} compares the PPNO predictions against the ground-truth solutions for both the field variable $u$ and its gradient $\nabla u_x$. The results show a remarkable correspondence between the predictions and the ground truth, highlighting the model's robust ability to generalize to structural complexities far beyond its training data.

\begin{figure}[H]
    \centering
    \includegraphics[width=1.0\linewidth]{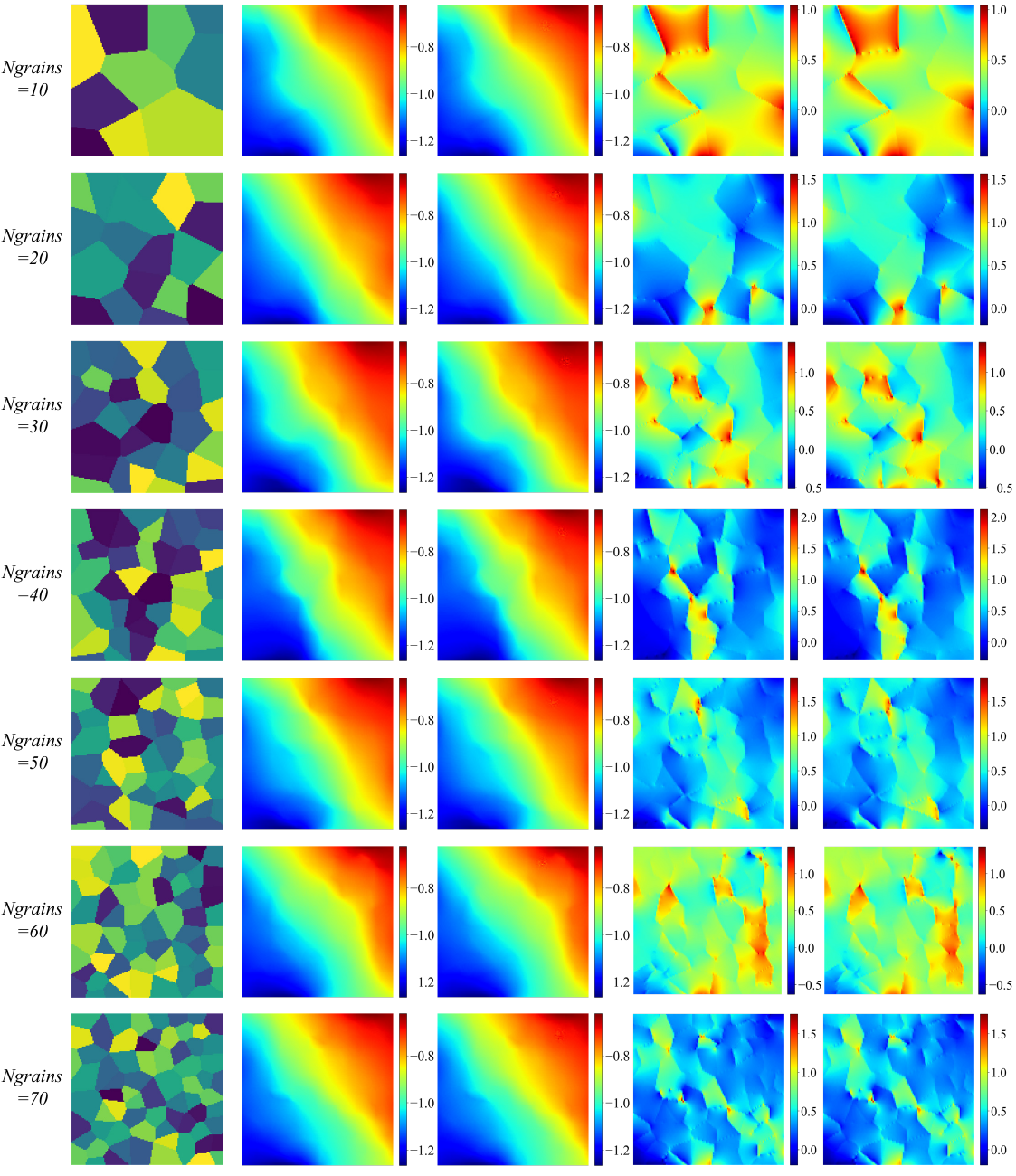}
    \caption{Generalization performance of the PPNO on out-of-distribution microstructures. The model was trained exclusively on 50-grain domains. The test cases shown feature varying grain counts and non-uniform sizes. The columns present, from left to right: (a) the input microstructure, where colors indicate unique grain orientations; (b) the ground-truth $u$ field and (c) the corresponding PPNO prediction; and (d) the ground-truth $\nabla u_x$ field and (e) its prediction. The high fidelity of the predictions underscores the model's robust generalization capabilities.}
    \label{fig:ap-test-u-ux}
\end{figure}

\begin{figure}[H]\ContinuedFloat
    \centering
    \includegraphics[width=1.0\linewidth]{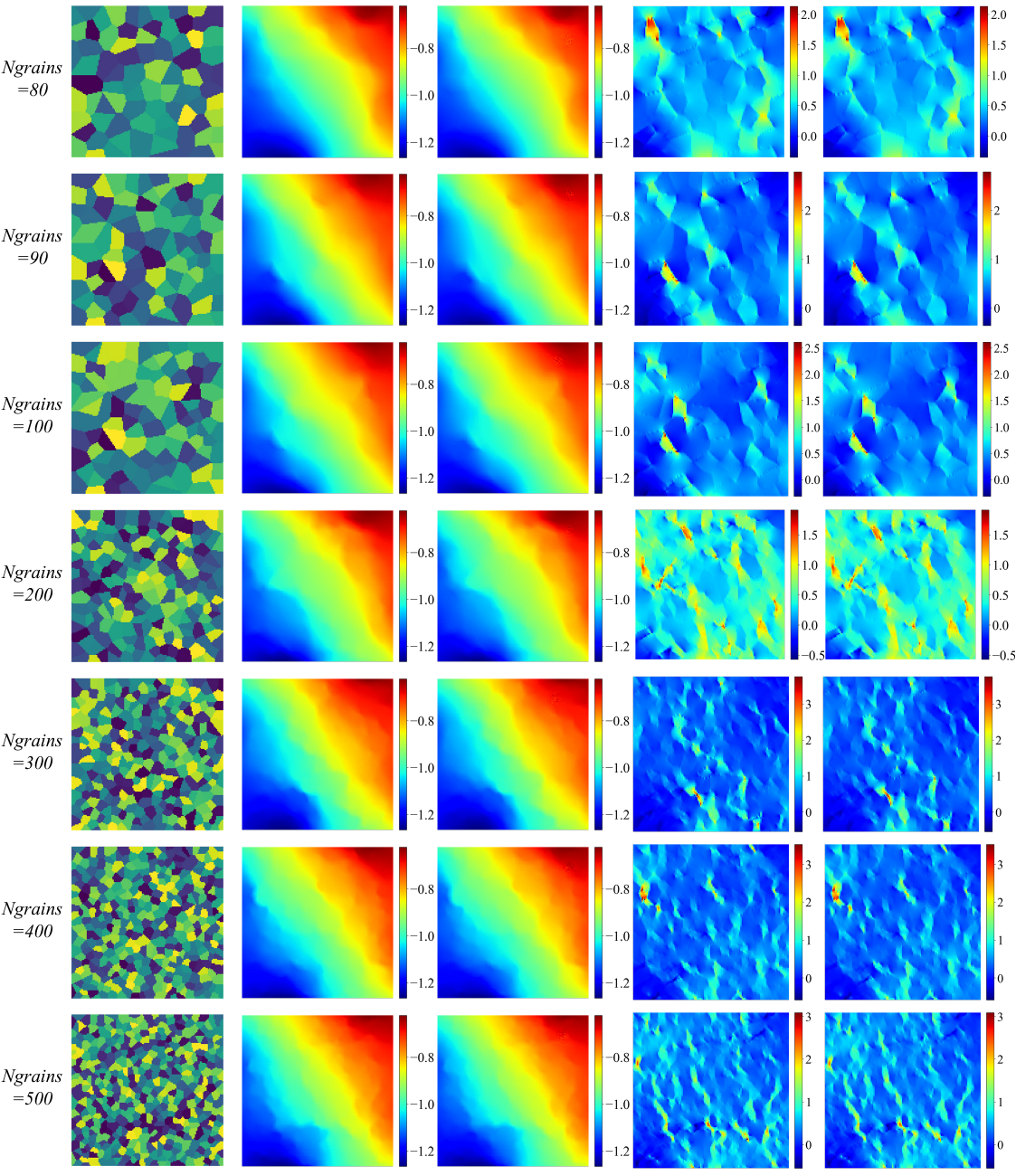}
    \caption{Generalization performance of the PPNO on out-of-distribution microstructures. The model was trained exclusively on 50-grain domains. The test cases shown feature varying grain counts and non-uniform sizes. The columns present, from left to right: (a) the input microstructure, where colors indicate unique grain orientations; (b) the ground-truth $u$ field and (c) the corresponding PPNO prediction; and (d) the ground-truth $\nabla u_x$ field and (e) its prediction. The high fidelity of the predictions underscores the model's robust generalization capabilities. (cont.)}
    \label{fig:ap-test-u-ux}
\end{figure}

\begin{figure}[H]\ContinuedFloat
    \centering
    \includegraphics[width=1.0\linewidth]{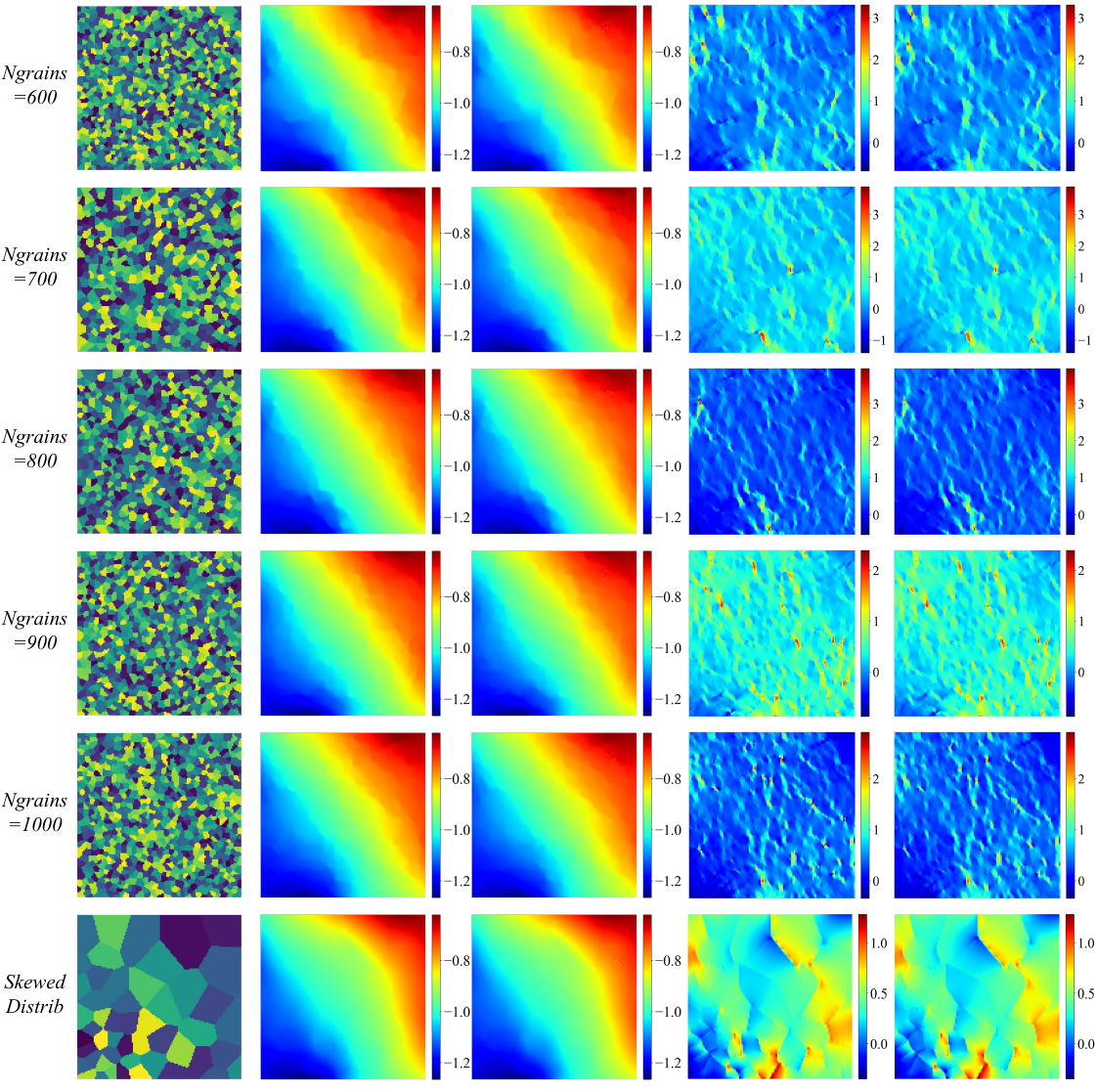}
    \caption{Generalization performance of the PPNO on out-of-distribution microstructures. The model was trained exclusively on 50-grain domains. The test cases shown feature varying grain counts and non-uniform sizes. The columns present, from left to right: (a) the input microstructure, where colors indicate unique grain orientations; (b) the ground-truth $u$ field and (c) the corresponding PPNO prediction; and (d) the ground-truth $\nabla u_x$ field and (e) its prediction. The high fidelity of the predictions underscores the model's robust generalization capabilities. (cont.)}
    \label{fig:ap-test-u-ux}
\end{figure}

\section{Rectangular domains with out-of-distribution microstructure patterns}
\label{sec:rectangular_domains}
To further evaluate the model's generalization capabilities, we extend the analysis to scenarios involving simultaneous variations in both domain geometry and microstructure patterns. We test the model on rectangular domains featuring microstructures distinct from the random Voronoi tessellations used during the pre-training phase. Specifically, we evaluat three out-of-distribution (OOD) patterns: graded Voronoi, hexagonal, and fiber-based microstructures (Types 2, 3, and 4, as detailed in Section \ref{sec:pretrain-out-of-distribution}). These patterns are defined on a rectangular domain $\Omega = [0,k_1L_0]\times[0,k_2L_0]$, where $L_0$ is the characteristic length of the square domains from the pre-training dataset, and the aspect ratios are set by $k_1 = 3.0625$ and $k_2 = 2.375$.

Figure \ref{fig:case-rectangle-vary} presents the domain decomposition analysis for these challenging cases. The results demonstrate that the solver successfully converges in all three scenarios, with the successive error diminishing to zero. This confirms the model's robust performance and its ability to generalize effectively to geometries and microstructures that were not encountered during training.

\begin{figure}[H]
    \centering
    \includegraphics[width=1.0\linewidth]{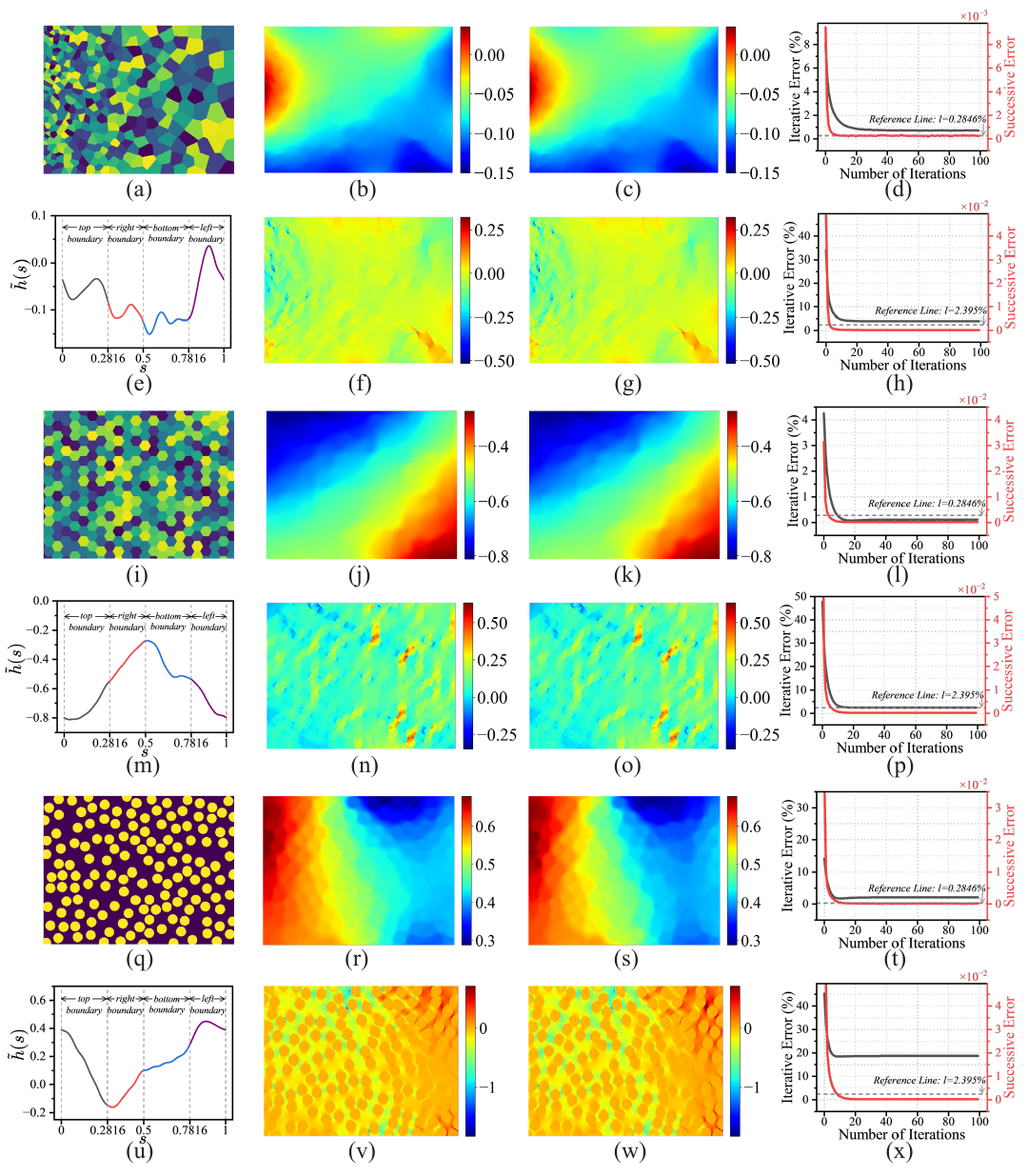}
    \caption{$4 \times 3$ rectangular domains with different microstructures, specifically varying grain sizes(a), hexagonal structures(i) and fibre-composite microstructure(q); (e,m,u) The boundary conditions applied; (b,j,r) Truth solution of the $u$ field; (c,k,s) Learning based solution of the $u$ field; (d,l,t) Error versus number of iterations of the $u$ field; (f,n,v) Truth solution of the $\nabla u$ field; (g,o,w) Learning based solution of the $\nabla u$ field; (h,p,x) Error versus number of iterations of the $\nabla u$ field.}
    \label{fig:case-rectangle-vary}
\end{figure}

\section{Algorithmic Robustness and Invariance}
\label{sec:algorithmic_robust}
To validate the robustness of the proposed algorithm, we conduct a systematic analysis of its performance on a rectangular domain of size $k_1L_0 \times k_2L_0$ ($k_1 = 3.0625$ and $k_2 = 2.375$). The study assesses the algorithm's sensitivity to variations in three fundamental aspects of the problem setup: (i) the microstructure topology, (ii) the applied boundary conditions, and (iii) the initial state of the solution field. Multiple independent trials are conducted for each category, with the final residual loss recorded at convergence.

Figure \ref{fig:robustness} visualizes the distribution of these final loss values. The results show that the loss consistently settles at a negligible magnitude, with minimal variance, regardless of the specific microstructure, boundary conditions, or initial guess used. This outcome demonstrates the algorithm's remarkable invariance to these factors, confirming its stability and reliability across a diverse range of simulation scenarios.

\begin{figure}[H]
    \centering
    \includegraphics[width=0.45\linewidth]{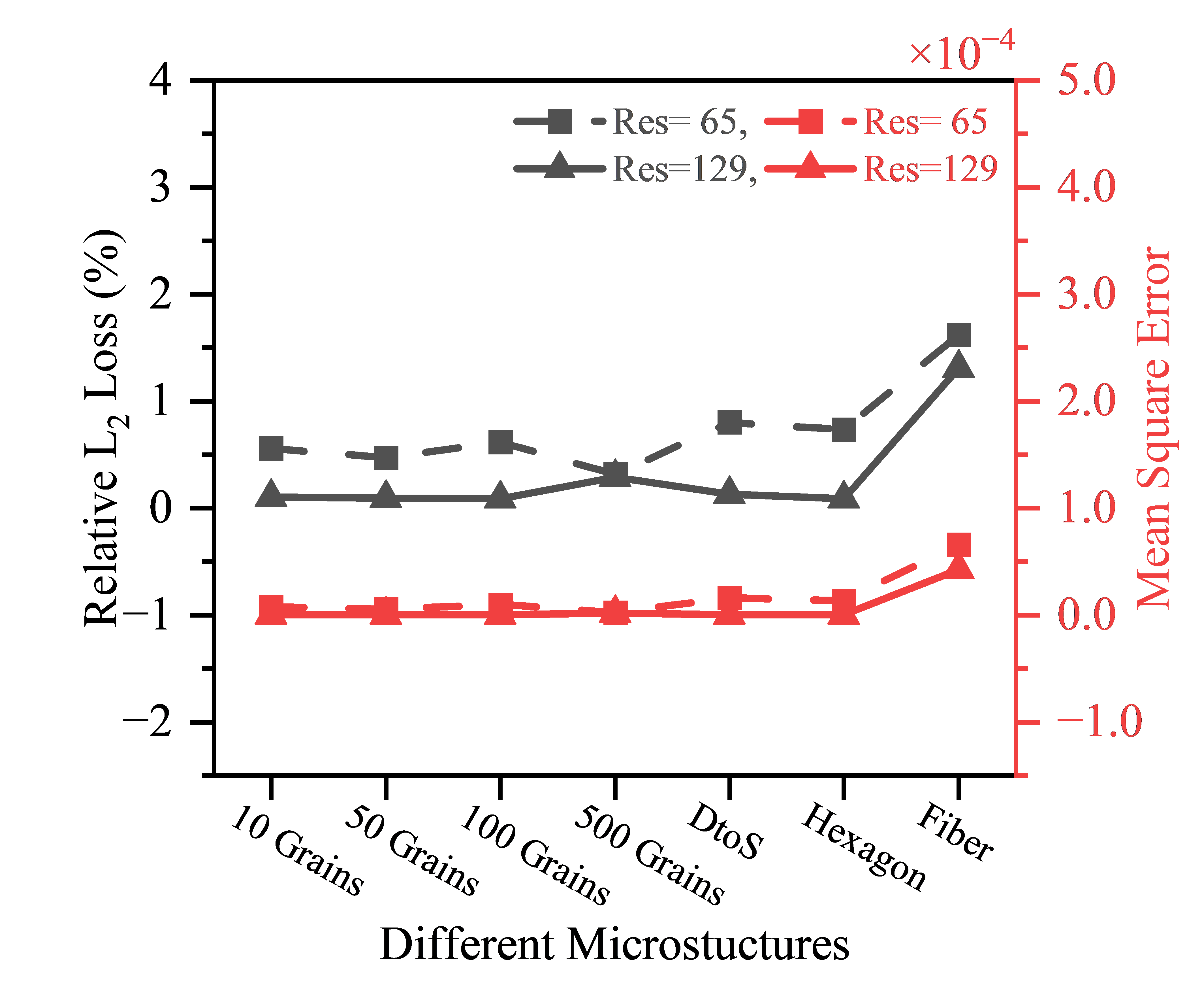}
    \includegraphics[width=0.45\linewidth]{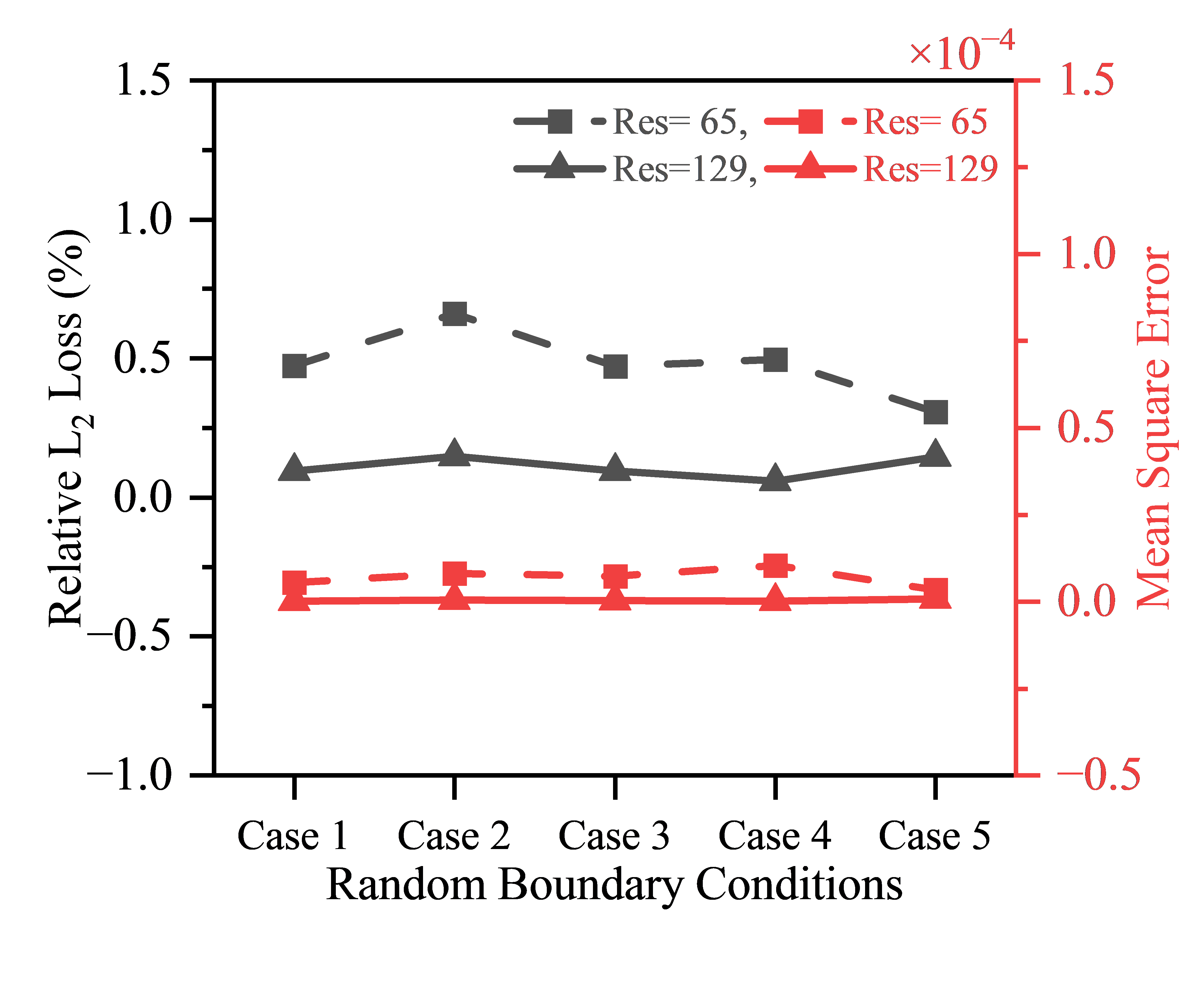}
    \includegraphics[width=0.45\linewidth]{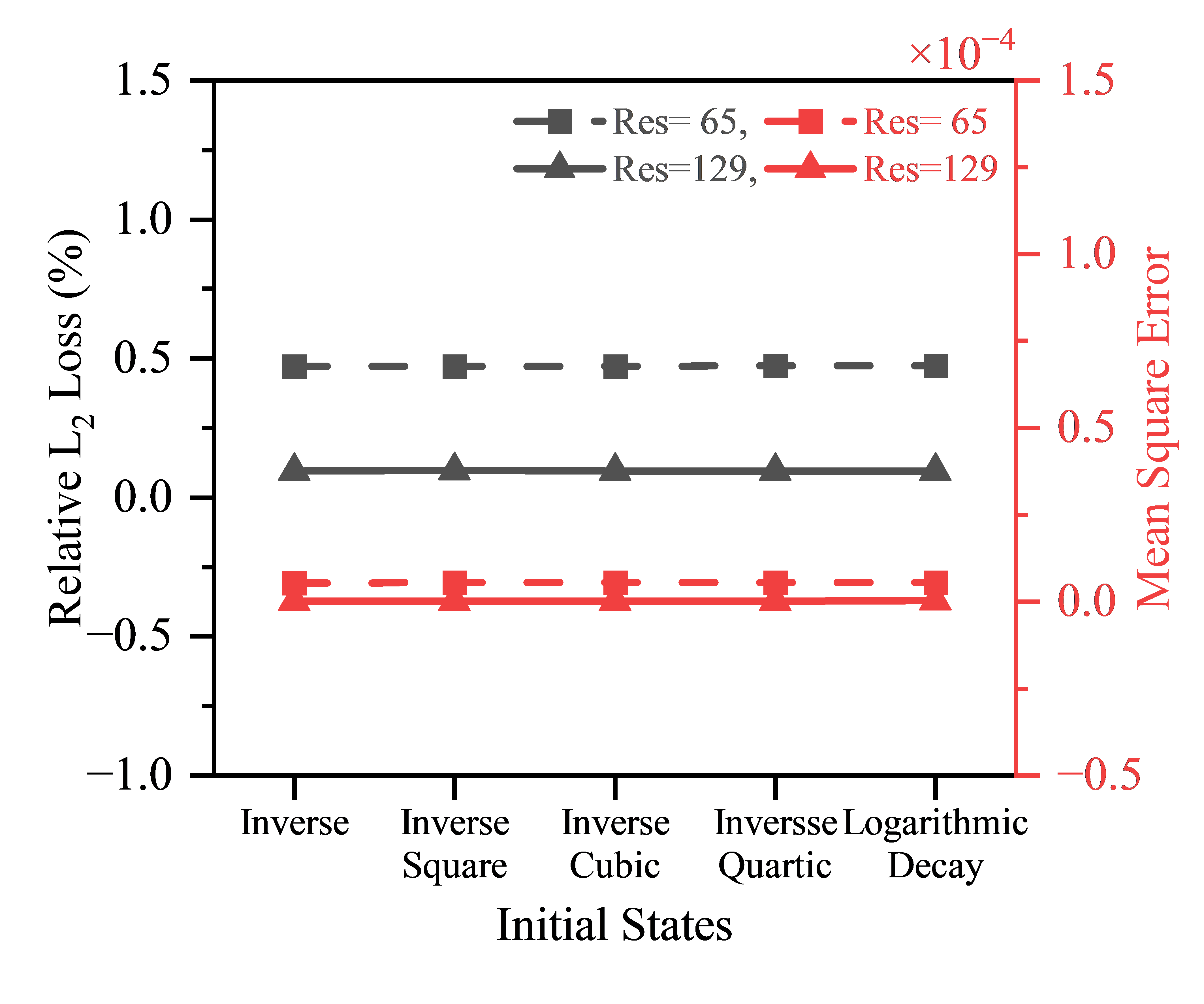}
    \caption{Illustration of the microstructure-invariance, boundary-condition-invariance, and the initial-state-invariance of the algorithm.}
    \label{fig:robustness}
\end{figure}

\section{Record of Boundary Parameters}
\label{boundary_params}
As specified in Equation \ref{eq:random_boundary}:
\begin{align}
    \tilde{g}(s) = & \sum_{n=1}^N \Bigg[\frac{1}{(n+1)^k} a_n \cos(2\pi n (s+s_0) + b_n) + \\& \frac{1}{(n+1)^k} c_n \sin(2\pi n (s+s_0) + d_n) \nonumber + e_n \Bigg], \quad s\in [0, 1),
\end{align}

For the case of the 10 by 10 square domain in section \ref{section_4.2}:

$s_0=0.1660$,

a = [0.5647,
    0.6172,
    0.5058,
    0.5864,
    0.9923,
    0.7664,
    0.9588,
    0.5977,
    0.7735,
    0.9409,
    0.7905,
    0.6819,
    0.7892,
    0.9299,
    0.9494];

b = [2.2939,
    0.9632,
    2.3627,
    2.1272,
    2.1576,
    1.2863,
    2.5159,
    1.8916,
    2.4059,
    2.9958,
    3.2609,
    3.7874,
    1.0620,
    3.7781,
    2.4301];

c = [0.7819,
    0.8284,
    0.8038,
    0.5504,
    0.6545,
    0.9865,
    0.8683,
    0.5224,
    0.8027,
    0.8873,
    0.8411,
    0.8136,
    0.7337,
    0.5664,
    0.9318];

d = [2.1178,
    1.9446,
    3.7620,
    1.7539,
    1.8797,
    3.4129,
    0.8300,
    1.4031,
    3.2959,
    3.6663,
    1.0197,
    3.8854,
    2.1395,
    1.3300,
    3.8331];

e = [-0.1347,
    0.0196,
   -0.0148,
    0.0468,
    0.0072,
    0.0975,
    0.1159,
    0.1680,
    0.2349,
    0.2215,
   -0.2226,
    0.1697,
    0.0274,
   -0.1462,
   -0.1749].


\end{document}